\documentclass[journal]{IEEEtran}
\usepackage{xcolor,colortbl}
\usepackage{algorithmic,algorithm}
\usepackage{graphicx}
\usepackage{balance}

\usepackage{cite}

\ifCLASSINFOpdf
\else
\fi
\usepackage{amssymb, amsmath, amsthm,bm,dsfont}
\DeclareFontFamily{OT1}{pzc}{}
\DeclareFontShape{OT1}{pzc}{m}{it}{<-> s * [1.10] pzcmi7t}{}
\DeclareMathAlphabet{\mathpzc}{OT1}{pzc}{m}{it}

\newcommand{\grad}[1]{\nabla_{#1^{\tran}} }
\newcommand{\w}{\bm{w}}

\newcommand{\eqdef}{\:\overset{\Delta}{=}\:}
\DeclareMathOperator*{\argmin}{argmin}
\newcommand{\Li}{\mathcal{L}_i}
\newcommand{\B}[2]{\mathcal{B}_{#1,i,#2}}
\newcommand{\tran}{{\sf T}}

\newtheorem{theorem}{Theorem}
\newtheorem{assumption}{Assumption}
\newtheorem{lemma}{Lemma}

\definecolor{Gray}{gray}{0.8}
\definecolor{LightCyan}{rgb}{0.88,1,1}

\begin{document}
%
\title{Federated Learning under Importance Sampling}
%
%
%

\author{Elsa~Rizk,~\IEEEmembership{}
        Stefan~Vlaski,~\IEEEmembership{Member,~IEEE,}
        and~Ali~H. Sayed,~\IEEEmembership{Fellow,~IEEE.}
\thanks{The authors are with the School of Engineering, École Polytechnique Fédérale de Lausanne, 1015 Lausanne, Switzerland (e-mail: elsa.rizk@epfl.ch; stefan.vlaski@epfl.ch; ali.sayed@epfl.ch).
A short conference article dealing with an earlier version of this work without extended arguments and proofs appears in \cite{DFL}.
\\ This work has been submitted for review.}
}

\maketitle

\begin{abstract}
Federated learning encapsulates distributed learning strategies that are managed by a central unit. Since it relies on using a selected number of agents at each iteration, and since each agent, in turn, taps into its local data, it is only natural to study optimal sampling policies for selecting agents and their data in federated learning implementations. Usually, only uniform sampling schemes are used. However, in this work, we examine the effect of importance sampling and devise schemes for sampling agents and data non-uniformly guided by a performance measure. We find that in schemes involving sampling without replacement, the performance of the resulting architecture is controlled by two factors related to data variability at each agent, and model variability across agents. We illustrate the theoretical findings with experiments on simulated and real data and show the improvement in performance that results from the proposed strategies.  
\end{abstract}


%
\IEEEpeerreviewmaketitle

\vspace{-0.2cm}
\section{Introduction}

\IEEEPARstart{I}{n} this work, we focus on algorithms that fall into the broad class of stochastic gradient descent (SGD). We consider a collection of $K$ heterogeneous agents that may have different computational powers. Each agent $k$ has locally $N_k$ data points, which we denote by $\{x_{k,n}\}$; the subscript $k$ refers to the agent, while the subscript $n$ denotes the sample index within agent $k$'s dataset. The goal of the agents is to find an optimizer for the aggregate risk function:
\begin{equation}\label{eq:GenProb-Uni}
	w^o \eqdef \argmin_{w \in \mathbb{R}^M} \frac{1}{K} \sum_{k=1}^K P_k(w),
\end{equation}
 where each $P_k(\cdot)$ is an empirical risk defined in terms of a loss function $Q_k(\cdot)$:
\begin{equation}
	P_k(w) \eqdef \frac{1}{N_k} \sum_{n=1}^{N_k} Q_k(w;x_{k,n}).
\end{equation}

Multiple strategies exist for solving such problems. They can be categorized into two main classes: a) partially decentralized strategies, which include a central process with access to all data and which controls the distribution of the data into the nodes for processing \cite{Duchi11,Zinkevich10,Bertsekas96anew}; and b) fully decentralized strategies, which consist of multiple agents connected by a graph topology and operating locally without oversight by a central processor \cite{sayed2014adaptation,chen2012limiting,Nedic09,Duchi12a}. Federated learning~\cite{mcmahan16,li2018federated,smith2017federated,caldas2018expanding,stattler19,konen2016federated,mohri2019agnostic,corinzia2019variational,khodak2019adaptive,chen2018federated,bonawitz2017practical,geyer2017differentially,mcmahan2017learning} offers a midterm solution, which consists of several agents collecting and processing local data that are then aggregated at the central processor. 

When implementing SGD, most strategies choose the samples according to a uniform distribution. In this work, we shall consider more general non-uniform sampling schemes, where the agents are sampled according to some distribution $\pi_k$ and the local data at agent $k$ are in turn sampled according to some other distribution $\pi_n^{(k)}$. In this setting, the central processor selects the subset of agents for processing according to $\pi_k$ and, once selected, an agent $k$ will sample its data according to $\pi_n^{(k)}$. The importance sampling process in this article therefore involves two layers. We use the superscript $(k)$ to denote the sampling distribution of the data at agent $k$. The sampling distributions $\{\pi_k, \pi_n^{(k)}\}$ are not fixed; instead we will show how to {\em adapt} them in order to enhance performance. At the same time, we will provide a detailed convergence analysis and establish performance limits. 

\subsection{Related Work}
Several works studied the convergence of the federated learning algorithm or distributed SGD under differing assumptions. These assumptions usually relate to the nature of the data (IID or non-IID), nature of the cost function (convex or non-convex), agent participation (full or partial), and operation (synchronous or asynchronous) \cite{jiang2018linear,khaled2019analysis,stich2018local,Wang2018CooperativeSA,Zhou_2018,yu2018parallel,wang2018adaptive,yu19,xie2019asynchronous,li2019convergence,li2018federated,smith2017federated,liu2019clientedgecloud}. Other works examine the convergence behavior of variations of the traditional FedAvg algorithm, such as FedProx \cite{li2018federated}, hierarchical version of FedAvg \cite{liu2019clientedgecloud},  multi-task federated learning \cite{smith2017federated}, and dynamic FedAvg \cite{DFL} -- see Table 1. 

\begin{table*}[htbp]
	\vspace*{-0.2cm}
	\caption{List of references on the convergence analysis of federated learning under different assumptions. This work along with our previous work \cite{DFL} are the only ones to tackle the 3 challenges of federated learning(non-IID data, asynchronous mode of operation, partial agent participation).}
	\vspace*{-0.4cm}
	\begin{center}
		\begin{tabular}{|c|c|c|c|c|c|c|}
			\hline
			\rowcolor{Gray}
			\textbf{References} & \textbf{Algorithm}& \textbf{Function Type}& \textbf{Data Heterogeneity} &\textbf{Operation} & \textbf{Agent Participation} & \textbf{Other Assumptions} \\
			\hline
			\cite{khaled2019analysis} & dist. gradient descent & convex &  \textbf{\textit{non-IID}} & synchronous & full & smooth\\
			\hline
			\cite{stich2018local} & dist. SGD & convex & IID & synchronous & full & smooth \\
			\hline
			\cite{Wang2018CooperativeSA, Zhou_2018} & dist. SGD & non-convex & IID & synchronous & full & smooth \\
			\hline
			\cite{yu2018parallel} & dist. SGD & non-convex & \begin{tabular}{@{}c@{}}
				\textbf{\textit{non-IID}} \\ IID
			\end{tabular} & \begin{tabular}{@{}c@{}}
				synchronous \\   \textbf{\textit{asynchronous}}
			\end{tabular} & full & - \\
			\hline
			\cite{wang2018adaptive} & dist. SGD & convex & \textit{\textbf{non-IID}} & synchronous & full & bounded gradients \\
			\hline
			\cite{yu19} & dist. momentum SGD & non-convex & \textbf{\textit{non-IID}} & synchronous & full & - \\
			\hline
			\cite{xie2019asynchronous} & FedAvg & \begin{tabular}{@{}c@{}}
				convex \\   some non-convex
			\end{tabular} & \textit{\textbf{non-IID}} & \textbf{\textit{asynchronous}} & full & -\\
			\hline
			\cite{li2019convergence} & FedAvg & convex & \textbf{\textit{non-IID}} & synchronous & \textbf{\textit{partial}} & bounded gradients
			\\ \hline
			\cite{li2018federated}  & FedProx & non-convex & \textit{\textbf{non-IID}} & \textbf{\textit{asynchronous}} &  \textbf{\textit{partial}} & - \\ \hline
			 \cite{liu2019clientedgecloud} & HierFAVG & \begin{tabular}{@{}c@{}}
			 	convex \\    non-convex
			 \end{tabular} &  \textit{\textbf{non-IID}} & synchronous & full & - \\ \hline
			 \cite{smith2017federated} & MOCHA & convex & \textit{\textbf{non-IID}} & synchronous & full & - \\ \hline
			\cite{DFL} & Dynamic FedAvg & convex & \textbf{\textit{non-IID}} & \textbf{\textit{asynchronous}} & \textbf{\textit{partial}} & model drift \\
			\hline
			\rowcolor{LightCyan} 
			\emph{this work} & ISFedAvg& convex & \textbf{\textit{non-IID}} & \textbf{\textit{asynchronous}} & \textbf{\textit{partial}} & importance sampling \\
			\hline
		\end{tabular}
		\label{tab:ref}
	\end{center}
\end{table*}

By contrast, not much work has been done on selection schemes for agents and data in federated learning. Given the architecture of a federated learning solution, this is a natural and important question to consider. The existing works in this domain can be split into two categories: those seeking better accuracy, and those seeking fairness. Of the works pertaining to the first category, reference \cite{nishio2019client} develops a new client selection scheme, called FedCS, where the goal of the central server is to choose as many agents as possible that can complete an iteration by a required deadline, after acquiring information about the agents' resources. 
Reference \cite{HybridFL} builds on this previous work to deal with non-IID data, and allows the server to collect some of the data from the agents and participate in the training of the model. The authors of \cite{nguyen2020fast} consider non-uniform sampling of agents and suggest approximate sampling probabilities that maximize the average inner product of the local gradient with the global gradient. References \cite{mohri19a,Li2020Fair} fall under the second category; in agnostic federated learning \cite{mohri19a}, the data distribution is assumed to be a mixture of the local distributions, and a minimax problem  for agent selection is solved. Reference \cite{Li2020Fair} generalizes the previous work by reweighting the cost function and assigning higher weights to agents with higher loss.

While there exist works that study the effect of importance sampling in distributed learning \cite{impSampSayed,alain2015variance,Nati14, zhao2015stochastic, jaggi17}, all of these works apply importance sampling to the data at each agent. To our knowledge, there are no works that examine the \textit{combined effect} of two hierarchical layers of sampling: one for the nodes and another for their data. By introducing a two-layer importance sampling scheme to the federated learning paradigm, we can tackle the problem of importance sampling both in relation to agents {\em and} also in relation to data.
\vspace{-0.1cm}
\subsection{Sampling and Inclusion Probabilities}
Before describing the problem setting, we need to clarify the difference between two notions: (a) sampling probability and (b) inclusion probability. Consider the following illustrative example. Consider $N=4$ balls of which we wish to choose $B=2$ balls non-uniformly and without replacement. Let the sampling probabilities be  $\pi_n = \{1/3, 1/6, 1/3,1/6\}$. This means that, initially, balls 1 and 3 are twice as likely to be selected compared to balls 2 and 4. For the first trial, all the inclusion probabilities are equal to the sampling probabilities, i.e., $\mathbb{P}(n \text{ chosen on $1^{st}$ trial}) = \pi_n$. However, since we are sampling \textit{without replacement}, the inclusion probabilities for the second trial depend on the outcome of the first trial, i.e., $\mathbb{P}(n \text{ chosen on $2^{nd}$ trial} | m$ chosen on $1^{st}$ trial$) = \pi_n /(1-\pi_m) $. Using the sampling probabilities, we can evaluate the likelihood that each ball will end up belonging to the selected set of $2$ balls. In particular, the probability that ball $1$ is chosen either in the first or second trial is given by: 
\begin{align}
	&\mathbb{P}(\text{1 chosen}) \notag \\
	 &= \sum_{n=2}^4 \mathbb{P}\big(\text{$1$ chosen on $1^{st}$ trial \& $n$ chosen on $2^{nd}$ trial)} \notag \\
	&\quad+\mathbb{P}\big(n \text{ chosen on $1^{st}$ trial \& $1$ chosen on $2^{nd}$ trial}) \notag \\
	 &= \sum_{n=2}^4 \pi_1  \frac{\pi_n}{1-\pi_1}+ \pi_n \frac{\pi_1}{1-\pi_n} .
\end{align}
Thus, the sampling probability is the working probability. It is the probability used to actually choose the samples, while the inclusion probability is a descriptive probability that indicates the likelihood of a ball being included in the final selected subset. Observe that the inclusion probabilities depend on the sampling scheme, while the sampling probabilities do not. When considering \textit{uniform sampling without replacement}, the inclusion probability is a multiple of the sampling probability. For example, sampling $B$ numbers from $\{1,2,\cdots,N\}$ with sampling probabilities $1/N$, the inclusion probability is found to be $\mathbb{P}(n \in \mathcal{B}) =  B/N$. Note further that while the sampling probabilities sum to 1 over all the sampling space, the inclusion probabilities $\mathbb{P}(n \in \mathcal{B})$ sum to $B$. In our derivations, we will be relying frequently on the inclusion probabilities.

Next, we consider a total number of $K$ agents. At each iteration $i$ of the algorithm, a subset of agents $\Li$ of size $L$ is chosen randomly \textit{without replacement}. We denote the probability that agent $k$ is included in the sample by $Lp_k$ \cite{horvitz1952generalization}, i.e.,
\vspace{-0.1cm}
\begin{equation}\label{eq:agentIncProb}
	p_k \eqdef \frac{\mathbb{P}\big(k \in \Li \big)}{L}.
\end{equation}
In addition, each sampled agent $k$ will run a mini-batch SGD by sampling $B_k$ data points $\mathcal{B}_{k,i}$ \textit{without replacement} from its local data. We denote the probability of inclusion of data point $n$ by $B_kp_n^{(k)}$, i.e.,
\vspace{-0.1cm}
\begin{equation}\label{eq:dataIncProb}
		p_n^{(k)} \eqdef \frac{\mathbb{P}\big(n \in \mathcal{B}_{k,i}\big)}{B_k}.
\end{equation}
We refer to $p_k$ and $p_n^{(k)}$ as the \textit{normalized inclusion probabilities}. They sum to $1$ over the sampling space; $p_k$ sums to $1$ over all agents and $p_n^{(k)}$ over the data at each agent.

\section{Algorithm Derivation}
The goal of the federated learning algorithm is to approximate the centralized solution $w^o$ while dealing with the constraint of distributed data. The goal is achieved by using an unbiased estimate of the gradient of the cost function, $	\frac{1}{K}\sum_{k=1}^K\grad{w}P_k(w).$
%
As explained in \cite{DFL} for the case of uniform sampling, if we assume each agent $k$ runs $E_k$ epochs per iteration $i$ (with each epoch using $B_k$ samples in $\B{k}{e}$), then we can construct an unbiased estimate for the true gradient by considering the following estimator:
\begin{equation}\label{eq:unbiasedEst}
	\frac{1}{L}\sum_{k \in \Li} \frac{1}{E_{k}B_{k}} \sum_{e=1}^{E_{k}}\sum_{b\in \B{k}{e}}\grad{w}Q_{k}(w;\bm{x}_{k,b}),
\end{equation} 
as opposed to the original estimator from \cite{mcmahan16}, where the main difference is the scaling by the epoch size $E_{k}$. This correction is important for the performance of the averaged model. Since the number of epochs $E_k$ can be non-uniform across the agents, then, without correction, agents with large epoch sizes will bias the solution by driving it towards their local model and away from $w^o$.


Expression \eqref{eq:unbiasedEst} is still not sufficient for our purposes in this article, since agents and data are allowed to be sampled \textit{non-uniformly without replacement}. In this case, we need to adjust \eqref{eq:unbiasedEst} by including the inclusion probabilities \cite{impSampSayed}. The inclusion probabilities are necessary to ensure the estimate is unbiased, as will later be seen in Lemma \ref{lemm:gradNoise}. The local estimate of the gradient at agent $k$ becomes $\frac{1}{Kp_k}\widehat{\grad{w}P_k}(w)$, with:
\begin{equation}\label{eq:genUnEst}
\widehat{\grad{w}P_k}(w) \eqdef   \frac{1}{E_kB_k}\sum_{e=1}^{E_k} \sum_{b\in \B{k}{e}} \frac{1}{N_k p_{b}^{(k)}} \grad{w}Q_k(w;\bm{x}_{k,b}).
\end{equation}
Motivated by \eqref{eq:genUnEst}, we can write down a stochastic gradient update at each agent $k$ at epoch $e$, and at the central processor at iteration $i$:
\begin{align}
	\w_{k,e} = &\: \w_{k,e-1} \notag \\
	&- \frac{\mu}{K p_{k} E_{k} B_{k}}\sum_{b\in \B{k}{e}} \frac{1}{N_{k}p_b^{(k)}}\grad{w}Q_{k}(\w_{k,e-1};\bm{x}_{k,b}), \label{eq:localUp}\\
	\w_i =& \frac{1}{L} \sum_{k \in \Li} \w_{k,E_{k}}	,\label{eq:comb}
\end{align}
where at each iteration $i$, step \eqref{eq:localUp} is repeated for $e=1,2,\cdots, E_k$. We arrive at the Algorithm \ref{alg:AFL}, which we refer to as Importance Sampling Federated Averaging (ISFedAvg).
\begin{algorithm}
	\begin{algorithmic}
		\caption{(Importance Sampling Federated Averaging)}\label{alg:AFL}
		\STATE{
			\textbf{initialize} $w_{0}$\;}
		\FOR{each iteration $i=1,2,\cdots$}\STATE{
			Select the set of participating agents $\Li$ by sampling \( L \) times from \( \{ 1, \ldots, K \} \) without replacement according to the sampling probabilities $\pi_k$.\\
			\FOR{each agent $k \in \mathcal{L}_i$} \STATE {
				\textbf{initialize} $\w_{k,0} = \w_{i-1}$ \\
				\FOR{each epoch $e=1,2,\cdots E_{k}$}\STATE{
					Find indices of the mini-batch sample \( \B{k}{e} \) by sampling \( B_{k} \) times from \( \{ 1, \ldots, N_{k} \} \) without replacement according to the sampling probabilities $\pi_n^{(k)}$.\\
					$\bm{g}  =\dfrac{1}{B_{k} } \sum\limits_{b\in \B{k}{ e}} \dfrac{1}{N_{k}p_b^{(k)}}\grad{w}Q_{k}(\w_{k,e-1};\bm{x}_{k,b})$ \\
					$\w_{{k},e} = \w_{{k},e-1} - \mu\dfrac{1}{E_{k}K p_{k}}\bm{g}$ \\
				}\ENDFOR
			} \ENDFOR
			\\ $\w_i = \dfrac{1}{L}\sum\limits_{k \in \Li} \w_{k,E_{k}}$
		}\ENDFOR
		
	\end{algorithmic}
\end{algorithm}

\vspace{-0.2cm}
\section{Convergence Analysis}
\subsection{Modeling Conditions}
To facilitate the analysis of the algorithm, we list some common assumptions on the nature of the local risk functions and their respective minimizers. Specifically, we assume convex cost functions with smooth gradients.

\begin{assumption}\label{assum:conLip}
	The functions $P_k(\cdot)$ are $\nu-$strongly convex, and $Q_k(\cdot;x_{k,n})$ are convex, namely:
	\begin{align}
		&P_k(w_2) \geq P_k(w_1) + \grad{w}P_k(w_1)(w_2-w_1) + \frac{\nu}{2}\Vert w_2-w_1\Vert^2, \\
		&Q_k(w_2;x_{k,n}) \geq Q_k(w_1;x_{k,n}) + \grad{w}Q_k(w_1;x_{k,n})(w_2-w_1).
	\end{align}
	 Also, the functions $Q_k(\cdot; x_{k,n})$ have $\delta-$Lipschitz gradients: 
	 \vspace{-0.1cm}
	\begin{align}
		\Vert \grad{w}Q_k(w_2;x_{k,n})-\grad{w}Q_k(w_1;x_{k,n})\Vert &\leq \delta \Vert w_2-w_1\Vert.
	\end{align}
\qed
\end{assumption}
\noindent We further assume that the individual minimizers $w_k^o = \argmin_{w \in \mathbb{R}^M} P_k(w)$ do not drift too far away from $w^o$.
\begin{assumption}\label{assum:bdLocalMin}
	The distance of each local model $w^o_k$ to the global model $w^o$ is uniformly bounded, $\Vert w^o_k - w^o \Vert \leq \xi.$
\qed
\end{assumption}

\subsection{Error Recursion}
Iterating the local update \eqref{eq:localUp} over multiple epochs and combining according to \eqref{eq:comb}, we obtain the following update for the central iterate:
\begin{align}\label{eq:recursion}
		\w_i = \: \w_{i-1} 	
		 -\mu \frac{1}{L} \sum_{k \in \Li}  &\frac{1}{K p_{k} E_{k}B_{k}} \sum_{e=1}^{E_{k}} \sum_{b \in \B{k}{e}}\frac{1}{N_kp_b^{(k)}} \notag\\ &\times\grad{w}Q_{k}(\w_{k,e-1};\bm{x}_{k,b}).
\end{align}
To simplify the notation, we introduce the error terms:
\begin{align}
	\bm{s}_i &\eqdef \frac{1}{L} \sum_{\ell \in \Li}\frac{1}{K p_{\ell}} \widehat{\grad{w}P_{\ell}}(\w_{i-1}) - \frac{1}{K}\sum_{k=1}^K \grad{w}P_k(\w_{i-1}), \label{eq:gradNoise}  \\
	\bm{q}_i &\eqdef \frac{1}{L} \sum_{\ell \in \Li} \frac{1}{K p_{\ell}E_{\ell}B_{\ell}} \sum_{e=1}^{E_{\ell}}\sum_{b\in\B{\ell}{e}} \frac{1}{N_{\ell} p_{b}^{(\ell)}}	\notag \\
	&\qquad\qquad \times \big( \grad{w}Q_k(\w_{\ell,e-1};\bm{x}_{\ell,b})  - \grad{w}Q_k(\w_{i-1};\bm{x}_{\ell,b})\big) \label{eq:incNoise}.
\end{align}
The first error term $\bm{s}_i$, which we call \textit{gradient error}, captures the error from approximating the true gradient by using subsets of agents and data; while, the second error term $\bm{q}_i$, which we call \textit{incremental error}, captures the error resulting from the incremental implementation, where at each epoch during one iteration, the gradient is calculated at the local iterate $w_{k,e-1}$. Note that this second error evaluates the loss function at the local and  global iterates. As we will show later, the incremental error will fade away, and the dominant factor will be the gradient error. Before establishing the main result in Theorem \ref{thrm:MT} on the convergence of ISFedAvg algorithm, we present preliminary results that will lead to it. Thus, to show the convergence of the algorithm, we must assure the gradient noise $\bm{s}_i$ has zero mean and bounded variance, and the incremental noise $\bm{q}_i$ has bounded variance. Furthermore, since we split the noise due to the stochastic gradient into incremental and gradient noise, we can split the analysis into that of the centralized steps and the local epochs. By proving that both the centralized and local steps converge, we show the global algorithm converges too. 

Replacing the two error terms \eqref{eq:gradNoise} and \eqref{eq:incNoise} into recursion \eqref{eq:recursion} and subtracting $w^o$ from both sides of the equation, we get the following error recursion:
\begin{equation}\label{eq:errRec}
	\widetilde{\w}_{i} = \widetilde{\w}_{i-1} + \mu \frac{1}{K}\sum_{k=1}^K \grad{w}P_k(\w_{i-1}) + \mu \bm{s}_i + \mu \bm{q}_i.
\end{equation}
To bound the $\ell_2-$norm of the error, we split it into two terms, centralized and incremental, using Jensen's inequality with some constant $\alpha \in (0,1)$ to be defined later:
\begin{align}\label{eq:l2NormErr}
	\Vert \widetilde{\w}_{i}\Vert^2 \leq& \frac{1}{\alpha} \left\Vert \widetilde{\w}_{i-1} +\mu\frac{1}{K}\sum_{k=1}^K \grad{w}P_k(\w_{i-1}) + \mu \bm{s}_i \right\Vert^2 \notag \\
	& +  \frac{1}{1-\alpha} \mu^2\Vert \bm{q}_i\Vert^2.
\end{align}
\noindent
We start with the first term that represents the centralized solution. We need to show that it converges. To do so, we start with the gradient noise, and we establish in Lemma \ref{lemm:gradNoise} that it remains bounded. We bound the gradient noise $\bm{s}_i$ under two constructions: sampling with replacement, and sampling without replacement. 

\begin{lemma}[\textbf{Estimation of first and second order moments of the gradient noise}]\label{lemm:gradNoise}
	The gradient noise defined in \eqref{eq:gradNoise} has zero mean:
	\begin{equation}
		\mathbb{E}\{\bm{s}_i | \w_{i-1}\} = 0,
\end{equation}		
	 with bounded variance, regardless of the sampling scheme. More specifically, sampling agents and data with replacement, results in the following bound:
	\begin{align}\label{eq:thrm-bd-var-w}
		\mathbb{E}\{\Vert \bm{s}_i\Vert^2 | \w_{i-1}\} \leq & \beta_s^2 \Vert \widetilde{\w}_{i-1} \Vert^2 + \sigma^2_s,
	\end{align}
	where $\widetilde{\w}_{i-1} = w^o - \w_{i-1}$ and the constants:
	\begin{align}\label{eq:thrm-cst-w}
		\beta_s^2 &\eqdef \frac{3\delta^2 }{L}+ \frac{1}{LK^2}\sum_{k=1}^K \frac{1}{p_k} \left( \beta_{s,k}^2 + 3\delta^2\right), \\
		\sigma_s^2 &\eqdef \frac{1}{LK^2}\sum_{k=1}^K \frac{1}{p_k}\left\{ \sigma_{s,k}^2 + \left( 3 + \frac{6}{E_kB_k}\right) \Vert \grad{w}P_k(w^o)\Vert^2 \right\}, \\
		\beta_{s,k}^2 & \eqdef  \frac{3\delta^2}{E_kB_k} \left( 1 + \frac{1}{N_k^2}\sum_{n=1}^{N_k} \frac{1}{p_n^{(k)}} \right), \label{eq:locCstBeta} 
	\end{align}
	\begin{align}
		\sigma_{s,k}^2 & \eqdef \frac{6}{E_kB_kN_k^2}\sum_{n=1}^{N_k}\frac{1}{p_n^{(k)}} \Vert \grad{w}Q_k(w^o;x_{k,n}) \Vert^2 \label{eq:locCstSig}.
	\end{align}
On the other hand, sampling agents and data without replacement results in the same bound but without the scaling by $L$ in the constants $\beta_s^2$ and $\sigma_{s}^2$.
\end{lemma}

\begin{proof}
	See Appendix \ref{app:gradNoise}.
\end{proof}

\noindent
The term $\sigma_{s,k}^2$ in the bound captures what we call \textit{data variability}. It is controlled by the mini-batch size $B_k$; as the mini-batch increases the effect of this term is reduced. The $\Vert \grad{w}P_k(w^o)\Vert $ term quantifies the suboptimality of the global model locally; we call its effect \textit{model variability}. It is reduced when the data and agents are more heterogeneous. From Assumption \ref{assum:bdLocalMin}, we can bound it uniformly.


Now that we have identified the mean and variance of the gradient noise, we can proceed to establish the important conclusion that the following centralized solution:
\begin{equation}\label{eq:centRecursion}
	\w_{i} = \w_{i-1} -\mu \frac{1}{L}\sum_{\ell \in \Li}\widehat{\grad{w}P_{\ell}(\w_{i-1})},
\end{equation}
converges exponentially to an $O(\mu)-$neighbourhood of the optimizer. In this implementation, the center processor aggregates the approximate gradients of the selected agents. We will subsequently call upon this result to examine the convergence behavior of the proposed federated learning solution.

\begin{lemma}[\textbf{Mean-square-error convergence of the centralized solution}]\label{lemm:convCentSol}
 	Consider the centralized recursion \eqref{eq:centRecursion} where the cost functions satisfy Assumption \ref{assum:conLip}, and where the first and second order moments of the gradient noise process satisfy the conditions in Lemma \ref{lemm:gradNoise}. Also, the samples are chosen without replacement. For step-size values satisfying  $\mu < 2\nu/(\delta^2 + \beta_s^2),$
 	it holds that $\mathbb{E}\Vert \widetilde{\w}_i\Vert^2$ converges exponentially fast according to the recursion \eqref{eq:thrm-MSD-rec}, where $\lambda = 1-2\mu\nu + \mu^2 (\delta^2+\beta_s^2)\in [0,1).$
 	\begin{equation}\label{eq:thrm-MSD-rec}
 		\mathbb{E}\Vert \widetilde{\w}_i\Vert^2 \leq \lambda \mathbb{E}\Vert\widetilde{\w}_{i-1}\Vert^2  + \mu^2 \sigma_s^2
 	\end{equation}
 	It follows from \eqref{eq:thrm-MSD-rec}  that, for sufficiently small step-sizes:
	\begin{equation}
		\mathbb{E}\Vert \widetilde{\w}_{i}\Vert^2 \leq \lambda^{i} \mathbb{E}\Vert \widetilde{\w}_{0}\Vert^2 + \frac{1-\lambda^i}{1-\lambda}\mu^2\sigma_{s}^2.
	\end{equation}
 \end{lemma}
\begin{proof}
	see Appendix \ref{app:convCentSol}.
	\end{proof}

We next bound the incremental noise $\bm{q}_i$. To do so, we introduce the local terms:
\begin{align}\label{eq:locGradNoise}
	\bm{q}_{k,i,e} \eqdef &\frac{1}{Kp_{k}} \Bigg( \frac{1}{B_k}\sum_{b \in \B{k}{e}}\frac{1}{N_k p_b^{(k)}}\grad{w}Q_k(\w_{k,e-1}; \bm{x}_{k,b}) \notag \\
	&- \grad{w}P_k(\w_{k,e-1}) \Bigg).
\end{align}
We show in the next lemma that the local gradient noise $\bm{q}_{k,i,e}$ has zero mean and bounded variance. This result is useful for showing that the local SGD steps converge in the mean-square-error sense towards their local models $w^o_k$. 
\begin{lemma}[\textbf{Estimation of first and second order moments of the local gradient noise}]\label{lem:locGradNoise}
	The local gradient noise defined in \eqref{eq:locGradNoise} has zero mean:
	\begin{equation}
		\mathbb{E}\left \{ \bm{q}_{k,i,e} \big| \mathcal{F}_{e-1}, \Li \right\} = 0,
	\end{equation}
	and bounded variance, regardless of the sampling scheme: 
	\begin{align}
		\mathbb{E} \left \{ \Vert\bm{q}_{k,i,e}\Vert^2 \big| \mathcal{F}_{e-1}, \Li \right\} \leq & \frac{E_k}{K^2p_k^2}\beta_{s,k} \Vert \widetilde{\w}_{k,e-1}\Vert^2 + \frac{1}{K^2p_k^2} \sigma_{q,k}^2,
	\end{align}
	where $\mathcal{F}_{e-1} = \{w_{k,0}, w_{k,1},\cdots, w_{k,e-1}\}$ is the filtration describing all sources of randomness due to the previous iterates, $\widetilde{\w}_{k,e} = w_k^o-\w_{k,e}$, and the constants are as defined in \eqref{eq:locCstBeta} and:
	\begin{equation}
		\sigma_{q,k}^2 = \frac{3}{B_kN_k^2}\sum_{n=1}^{N_k}\frac{1}{p_n^{(k)}}\Vert \grad{w}Q_k(w_k^o; x_{k,n})\Vert^2.
	\end{equation}
\end{lemma}
\begin{proof}
	see Appendix \ref{app:locGradNoise}.
\end{proof}

Now that we have showed that the local gradient noise of the incremental step has bounded variance, we can study the mean square deviation of the local SGD. 

\begin{lemma}[\textbf{Mean-square-error convergence of the local incremental step}]\label{lem:von-loc-inc-step}
	For every agent $k$, consider the local stochastic gradient recursion \eqref{eq:localUp} where the cost function is subject to Assumption \ref{assum:conLip}, and where the first and second order moments of the gradient noise process satisfy the conditions in Lemma \ref{lem:locGradNoise}. For step-size values satisfying:
	\begin{equation}
		\mu < \frac{2\nu}{\delta^2 + \frac{E_k}{K^2p_k^2}\beta_{s,k}^2},
	\end{equation}
	it holds that $\mathbb{E}\Vert \widetilde{\w}_{k,e}\Vert^2$ converges exponentially fast according to the recursion:
	\begin{equation}\label{eq:thrm-loc-MSD-rec}
		\mathbb{E}\Vert \widetilde{\w}_{k,e}\Vert^2 \leq \lambda_{k} \mathbb{E}\Vert \widetilde{\w}_{k,e-1}\Vert^2  + \mu^2 \sigma_{q,k}^2,
	\end{equation}
	where:
	\begin{equation}
		\lambda_k =1-2\nu\mu + \mu^2 \left(\delta^2 + \frac{E_k}{K^2p_k^2}\beta^2_{s,k}\right) \in [0,1).
	\end{equation}
	It follows from \eqref{eq:thrm-loc-MSD-rec} that, for sufficiently small step-sizes:
	\begin{align}
		\mathbb{E}\Vert \widetilde{\w}_{k,e}\Vert^2 \leq \lambda_{k}^{e} \mathbb{E}\Vert \widetilde{\w}_{k,0}\Vert^2 + \frac{1-\lambda_k^e}{1-\lambda_k}\mu^2\sigma_{q,k}^2.
	\end{align}
\end{lemma}
\begin{proof}
	see Appendix \ref{app:von-loc-inc-step}
\end{proof}
\noindent
We can finally bound the incremental noise in the following lemma.
\begin{lemma}[\textbf{Estimation of the second order moments of the incremental noise}]\label{lemm:incNoise}
	The incremental noise defined in \eqref{eq:incNoise} has bounded variance:
	\begin{align}
		\mathbb{E}\Vert \bm{q}_i \Vert^2 &\leq  O(\mu) \mathbb{E}  \Vert \widetilde{\w}_{i-1}\Vert ^2  + O(\mu) \xi^2 + O(\mu^2)\frac{1}{K}\sum_{k=1}^K\sigma_{q,k}^2,
	\end{align}	
	where the $O(\cdot)$ terms depend on epoch sizes, local convergence rates, total number of data samples, number of agents, Lipschitz constant, and data and agent normalized inclusion probabilities. Thus, $\mathbb{E}\Vert \bm{q}_i\Vert^2 = O(\mu)$.
\end{lemma}

\begin{proof}
	see Appendix \ref{app:incNoise}.
\end{proof}
\noindent
We observe an average data variability term across agents $\sigma_{q,k}^2$, and a model variability term $\xi^2$. However, the effect of the latter dominates since it is multiplied by an $O(\mu)$ term as opposed to $O(\mu^2)$.

\subsection{Main Theorem}
Now that we have bounded each term of \eqref{eq:l2NormErr} and using Lemma \ref{lemm:convCentSol} and \ref{lemm:incNoise}, we find:
	\begin{align}
		&\mathbb{E}\Vert \widetilde{\w}_i\Vert^2 \notag \\
		&\leq \frac{1}{\alpha}\left( \lambda \mathbb{E}\Vert  \widetilde{\w}_{i-1}\Vert^2 + \mu^2 \sigma_s^2\right) + \frac{O(\mu^3)}{1-\alpha}\left( \mathbb{E}\Vert  \widetilde{\w}_{i-1}\Vert^2 + \xi^2 \right)\notag\\
		&\quad + \frac{O(\mu^4)}{1-\alpha}\frac{1}{K}\sum_{k=1}^K  \sigma_{q,k}^2 , \notag \\
		&= \lambda' \mathbb{E}\Vert \widetilde{\w}_{i-1} \Vert^2 + \frac{\mu^2 \sigma_s^2}{\alpha}  + \frac{O(\mu^3)}{1-\alpha}\xi^2 + \frac{O(\mu^4)}{1-\alpha}\frac{1}{K}\sum_{k=1}^K\sigma_{q,k}^2,
	\end{align}
where:
\begin{align}
	\lambda' &\eqdef \frac{\lambda}{\alpha} + \frac{O(\mu^3)}{1-\alpha}, \notag \\
	&= \frac{1-2\mu\nu + \mu^2(\delta^2 +\beta_s^2)}{\alpha} + \frac{O(\mu^3)}{1-\alpha}\notag \\
	&= O\left(\frac{1}{\alpha}\right)  + O\left(\frac{\mu}{\alpha}\right)+O\left(\frac{\mu^2}{\alpha}\right) + O\left(\frac{\mu^3}{1-\alpha}\right).
\end{align}
Applying this bound recursively we obtain:
\begin{align}
	\mathbb{E}\Vert \widetilde{\w}_{i}\Vert^2 
	&\leq \left(\lambda'\right)^i \mathbb{E}\Vert \widetilde{\w}_0 \Vert^2 + \frac{1-\left(\lambda'\right)^i}{1-\lambda'} \Bigg( \frac{\mu^2 \sigma_s^2}{\alpha} +   \frac{O(\mu^3)}{1-\alpha} \xi^2 \notag \\
		&\quad + \frac{O(\mu^4)}{1-\alpha}\frac{1}{K}\sum_{k=1}^K \sigma_{q,k}^2 \Bigg),
\end{align}
and then taking the limit $i \to \infty$:
\begin{align}
	&\lim_{i\to \infty} \mathbb{E}\Vert \widetilde{\w}_{i}\Vert^2 \notag \\ 
	&\leq \frac{1}{1-\lambda'} \left(\frac{ \mu^2 \sigma_s^2 }{\alpha}  + \frac{O(\mu^3)}{1-\alpha}\xi^2+\frac{O(\mu^4)}{1-\alpha} \frac{1}{K}\sum_{k=1}^K  \sigma_{q,k}^2 \right).
\end{align}
for $\lambda' < 1$, which for $\alpha = \sqrt{\lambda}$  is achieved when:
\begin{align}
	&\mu < \min \left \{ \frac{2\nu}{\delta^2 + \beta_s^2}, \frac{2\nu}{\delta^2 + \frac{E_k}{K^2p_k^2}\beta_{s,k}^2} \right \}, \\
	& O(\mu^3) < (1-\sqrt{\lambda})^2.
\end{align}
Thus, since $\alpha = O(1)$, $1-\alpha = O(\mu)$ and $1-\lambda'  = O(\mu)$:
\begin{align}
	\lim_{i\to\infty}\mathbb{E}\Vert \widetilde{\w}_{i}\Vert^2 \leq  O(\mu) \sigma_s^2 + O(\mu)\xi^2 + O(\mu^2)\frac{1}{K}\sum_{k=1}^K \sigma_{q,k}^2  .
\end{align} 
The result is summarized in the theorem.
\begin{theorem}[\textbf{Mean-square-error convergence of federated learning under importance sampling}]\label{thrm:MT} Consider the iterates $\w_i$ generated by the importance sampling federated averaging algorithm. For sufficiently small step-size $\mu$, it holds that the mean-square-error converges exponentially fast: 
	\begin{align}
		\mathbb{E}\Vert \widetilde{ \w}_i \Vert^2 \leq & O\left((\lambda')^i\right) + O(\mu) \left(\sigma_s^2  + \xi^2\right)+ O(\mu^2) \frac{1}{K}\sum_{k=1}^K \sigma_{q,k}^2, 
	\end{align}
	where $\lambda' = 1- O(\mu) + O(\mu^2)  \in [0,1).$
	\qed
\end{theorem}
\section{Importance Sampling}

Due to the heterogeneity of nodes which arise from their data and computational capabilities, it is important to guide the algorithm based on the potential contribution that each agent can have on the overall performance. By allowing asynchronicity, i.e., different epoch sizes among agents, we can take advantage of the varying computational capabilities. From \cite{impSampSayed}, we know that the choice of samples at each iteration affects the solution. Therefore, instead of choosing the samples uniformly, we consider importance sampling where samples are chosen according to some distribution to be determined. A similar scheme can be enforced on the participating agents. In what follows, we show that using importance sampling enhances the overall performance.

\subsection{Agent Level: Importance Sampling of Data}
Every agent $k$ at each epoch must select a mini-batch of data based on the normalized inclusion probabilities $p_n^{(k)}$. To find the optimal probabilities, we minimize the bound on the variance of the local gradient noise $\sigma_{s,k}^2$.
We solve the problem for both cases when sampling is done with replacement and without replacement. The results are the same for both sampling schemes.
\begin{lemma}[\textbf{Optimal local data inclusion probabilities}]\label{lem:optPD}
The optimal local data normalized inclusion probabilities are given by:
	\begin{equation}\label{eq:dataOptP}
		 p_n^{(k),o} \eqdef \frac{\Vert \grad{w}Q_k(w^o;x_{k,n})\Vert}{ \sum_{m=1}^{N_k} \Vert \grad{w}Q_k(w^o;x_{k,m})\Vert}.
	\end{equation}
\end{lemma}

\begin{proof}
	By introducing a Lagrange multiplier $\lambda$, the optimization problem can be reformulated as:
	\begin{equation}
		 \min_{p_n^{(k)},\lambda} \sum_{n=1}^{N_k}\frac{1}{p_n^{(k)}} \Vert \grad{w}Q_k(w^o;x_{k,n})\Vert^2 + \lambda\left( \sum_{n=1}^{N_k} p_n^{(k)} -1 \right).
	\end{equation}
	Then, taking the derivative with respect to $p_n^{(k)}$ and setting it to zero we get \eqref{eq:pfOptDP}. 	Next, substituting $p_n^{(k),o}$ into the condition, we find \eqref{eq:pfOptDp-lam}.
	\begin{align}
		p_n^{(k),o} &= \frac{ \Vert \grad{w}Q_k(w^o;x_{k,n})\Vert}{\sqrt{\lambda}} \label{eq:pfOptDP} \\
		\sqrt{\lambda} &= \sum_{m=1}^{N_k} \Vert \grad{w}Q_k(w^o;x_{k,m})\Vert \label{eq:pfOptDp-lam}
	\end{align}
\end{proof}
\vspace{-0.1cm}
\noindent
As seen in Lemma \ref{lem:optPD}, more weight is given to a data point that has a greater gradient norm, thus increasing its chances of being sampled and resulting in a faster convergence rate. In addition, we observe that the more homogeneous the data is the more uniform the inclusion probability is.

\subsection{Cloud Level: Importance Sampling of Agents}
At each iteration, the cloud must select a subset of agents to participate. The agents are selected in accordance with the normalized inclusion probabilities $p_k$. To find the optimal probabilities, we minimize the bound on the variance of the gradient noise $\sigma_{s}^2$
The following result holds for sampling with and without replacement, since the gradient noise only differ by a multiplicative factor.

\begin{lemma}[\textbf{Optimal agent inclusion probabilities for sampling with replacement}]\label{lem:optPKw}
	The optimal agent normalized inclusion probabilities are given by \eqref{eq:agentOptP}, where $\alpha_k  = \left( 3 + \frac{6}{E_kB_k} \right)$:
	\begin{equation}\label{eq:agentOptP}
		 p_k^{o} \eqdef  \frac{\sqrt{\sigma_{s,k}^2 +\alpha_k \Vert \grad{w}P_k(w^o)\Vert^2}}{ \sum_{\ell=1}^{K} \sqrt{\sigma_{s,\ell}^2 + \alpha_{\ell}\Vert \grad{w}P_{\ell}(w^o)\Vert^2}}.
	\end{equation}
\end{lemma}

\begin{proof}
The proof follows similarly to that of Lemma \ref{lem:optPD}.
\end{proof}
\noindent
We observe that the normalized inclusion probabilities will be closer to a uniform distribution the more the data and model variability terms are similar across agents. 

\subsection{Practical Issues}
In the previous subsections, we focused on finding the optimal inclusion probabilities for the agents and data. However, several practical issues arise. The first is that all probabilities are calculated based on the optimal model $w^o$, which we do not have access to. To overcome this issue, we estimate the probabilities at each iteration by calculating them according to the current model $\w_{i-1}$. Thus,

\begin{align}
	\widehat{p}_{n}^{(k),o} &= \frac{\Vert \grad{w}Q_k(\w_{i-1};x_{k,n})\Vert}{\sum_{m=1}^{N_k}\Vert \grad{w}Q_k(\w_{i-1};x_m)\Vert},  \label{eq:approxPn1} \\
	\widehat{p}_k^o &= \frac{\sqrt{\sigma^2_{s,k} + \alpha_k \Vert \grad{w}P_k(\bm{w}_{i-1})\Vert^2}}{\sum_{\ell=1}^K \sqrt{\sigma^2_{s,\ell} + \alpha_{\ell} \Vert \grad{w}P_{\ell}(\bm{w}_{i-1})\Vert^2}} . \label{eq:approxPk1}
\end{align}
In addition, since calculating the true gradient of the local loss function is costly, we replace it with the mini-batch approximation when calculating $p_k^o$:
\begin{equation}
	\widehat{p}_k^o = \frac{\sqrt{\sigma^2_{s,k} + \alpha_k \left\Vert \widehat{\grad{w}P_k}(\bm{w}_{i-1}) \right\Vert^2}}{\sum_{\ell=1}^K \sqrt{\sigma^2_{s,\ell} + \alpha_{\ell}\left\Vert \widehat{\grad{w}P_{\ell}}(\bm{w}_{i-1} )\right\Vert^2}}  .
\end{equation}
Furthermore, every agent has access to all of its data and consequently to all of the gradients. However, the cloud does not have access to the gradients of all agents, and in turn cannot calculate the denominator of $p_k$. Instead, we propose the following solution: at iteration 0, all probabilities are set to $p_k = \frac{1}{K}$; then, during the $i^{th}$ iteration, after the participating agents $\ell \in \Li$ send the cloud their stochastic gradients $\widehat{\grad{w}P_{\ell}}(\w_{i-1})$, the probabilities are updated as follows: 
\begin{align}\label{eq:approxPk}
	\widehat{p}_{k}^o = &\frac{\sqrt{\sigma^2_{s,k} + \alpha_k\left\Vert \widehat{\grad{w}P_k}(\bm{w}_{i-1}) \right\Vert^2}}{\sum_{\ell \in \Li} \sqrt{\sigma^2_{s,\ell} + \alpha_{\ell} \left\Vert \widehat{\grad{w}P_{\ell}}(\bm{w}_{i-1}) \right\Vert^2}}   \left( 1 - \sum_{\ell \in \Li^c}  \widehat{p}_{\ell}^o \right),
\end{align}
where the multiplicative factor follows from ensuring all the probabilities $\widehat{p}_k^o$ sum to $1$.
Similarly for the local probabilities, since we are implementing mini-batch SGD, we only update the probabilities of the data points that were sampled:
\begin{align}\label{eq:approxPn}
	\widehat{p}_n^{(k),o} = &\frac{\Vert \grad{w}Q_k(\w_{i-1};x_{k,n})\Vert}{\sum_{b \in \B{k}{e}}\Vert \grad{w}Q_k(\w_{i-1};x_b)\Vert} \left( 1 - \sum_{b \in \B{k}{e}^c}\widehat{p}_b^{(k),o} \right).
\end{align}

Finally, the last problem arises when sampling without replacement. We have found the optimal inclusion probabilities and not the optimal sampling probabilities, and moving from the former to the latter is not trivial. Thus, we rely on the literature under sampling without replacement with unequal probabilities. Multiple sampling schemes exist such that the sampling probabilities do not need to be calculated explicitly. In general, there are multiple non-uniform sampling without replacement schemes that guarantee the same inclusion probabilities. We choose to implement the sampling scheme proposed in \cite{hartley1962}, which ensures the inclusion probabilities are $p_k^{o}$ and $p_n^{(k),o}$ for the agents and data, respectively. More explicitly, we first calculate the \textit{progressive totals of the inclusion probabilities $\Pi_k = \sum_{\ell=1}^k Lp_k$, for $k=1,2,\cdots,K$, and we set $\Pi_0 = 0$. Then, we select uniformly at random a \textit{uniform variate} $d \in [0,1)$. Then, we select the $L$ agents that satisfy $\Pi_{k-1} \leq q+ \ell < \Pi_{k}$, for some $\ell = 0,1,\cdots,L-1$.
}

\section{Experimental Section}
To validate the theoretical results, we devise two experiments. The first consists of simulated data with quadratic risk functions, and the second consists of a real dataset with logistic risk functions.

\subsection{Regression}
We first validate the theory on a regression problem. We consider $K = 300$ agents, for which we generate $N_k = 100$ data points for each agent $k$ as follow: Let $\bm{u}_{k,n}$ denote an independent streaming sequence of two-dimensional random vectors with zero mean and covariance matrix $R_{u_k} = \mathbb{E}\:\bm{u}_{k,i}\bm{u}_{k,i}^{\tran}$. Let $\bm{d}_k(n)$ denote a streaming sequence of random variables that have zero mean and variance $\sigma_{d_k}^2 = \mathbb{E}\:\bm{d}_k^2(n)$. Let $r_{d_ku_k} = \mathbb{E}\:\bm{d}_k(n)\bm{u}_{k,n}$ be the cross-variance vector. The data $\{\bm{d}_k(n),\bm{u}_{k,n}\}$ are related by the following linear regression model:
\begin{equation}
	\bm{d}_k(n) = \bm{u}_{k,n}w^{\star} + \bm{v}_k(n),
\end{equation}
for some randomly generated parameter vector $w^{\star}$ and where $\bm{v}_k(n)$ is a zero mean white noise process with variance $\sigma_{v_k}^2 = \mathbb{E}\bm{v}_k^2(n)$, independent of $\bm{u}_{k,n}$. The local risk is given by:
\begin{equation}
	P_k(w) = \frac{1}{N_k}\sum_{n=1}^{N_k} \Vert \bm{d}_k(n) - \bm{u}_{k,n}^{\tran}w\Vert^2 + \rho \Vert w\Vert^2.
\end{equation}
We set $\rho = 0.001$, while the batch sizes $B_k$ and the epoch sizes $E_k$ are chosen uniformly at random from the range $[1,10]$ and $[1,5]$, respectively. During each iteration, there are $L = 6$ active agents.
To test the performance of the algorithm, we calculate at each iteration the mean-square-deviation (MSD)  of the parameter vector $\bm{w}_i$ with respect to the true model $w^o$:
\begin{equation}
	\mbox{\rm MSD}_i = \Vert \bm{w}_i - w^o\Vert^2.
\end{equation}
The optimization problem has the closed form expression: 
\begin{equation}
	w^o = \left( \widehat{R}_{u} + \rho I\right)^{-1} \widehat{R}_u w^{\star} +\left( \widehat{R}_{u} + \rho I\right)^{-1} \widehat{r}_{uv},
\end{equation}
where:
\begin{align}
	\widehat{R}_u &\eqdef \frac{1}{K}\sum_{k=1}^K \frac{1}{N_k}\sum_{n=1}^{N_k} \bm{u}_{k,n}^\tran \bm{u}_{k,n}, \\
	\widehat{r}_{uv} &\eqdef  \frac{1}{K}\sum_{k=1}^K \frac{1}{N_k}\sum_{n=1}^{N_k} \bm{v}_k(n)\bm{u}_{k,n} . 
\end{align}

We run four tests: we first run the standard FedAvg algorithm where the mini-batches are chosen uniformly with replacement. We then run Algorithm \ref{alg:AFL}, once with the optimal probabilities $p_n^{(k)}$ and $p_k$ in \eqref{eq:dataOptP}--\eqref{eq:agentOptP}, once with the approximate probabilities \eqref{eq:approxPn1}--\eqref{eq:approxPk1}, and once with \eqref{eq:approxPk}--\eqref{eq:approxPn}. We implement the sampling scheme from \cite{hartley1962}. We set the step-size $\mu = 0.01$. Each test is repeated 100 times, and the resulting MSD is averaged. We get the curves as shown in Figure \ref{fig:regMSD}. We see that the importance sampling scheme does better than the standard sampling algorithm. This comes as no surprise, since the probabilities were chosen to minimize the bound on the MSD. The importance sampling scheme (green curve) improved the MSD bound by 23.1 dB compared to the standard federated learning scheme (blue curve). Furthermore, we observe that approximate probabilities do not degrade the performance of the algorithm. Our proposed approximate solution \eqref{eq:approxPk}--\eqref{eq:approxPn} (purple curve) performs just as well as using the true probabilities. In fact, we observe that the approximate probabilities converge to the true ones, $\Vert p_k^o - \widehat{p}_k^o\Vert = 1.22e-2$ and $\frac{1}{K}\sum_{k=1}^K\Vert p_n^{(k),o} - \widehat{p}_n^{(k),o} \Vert= 1.54e-2$. Similarly, the approximate probabilities \eqref{eq:approxPk1}--\eqref{eq:approxPn1}  do not degrade the overall performance. They, in fact, outperform the other solutions and converge faster (red curve). This is not surprising, since, at each iteration, we are attributing higher probabilities to agents and data points that have greater gradients. We are increasing their chances of being selected and thus taking steeper steps towards the true model. 

\begin{figure}[h!]
\begin{center}
	\includegraphics[scale=0.5]{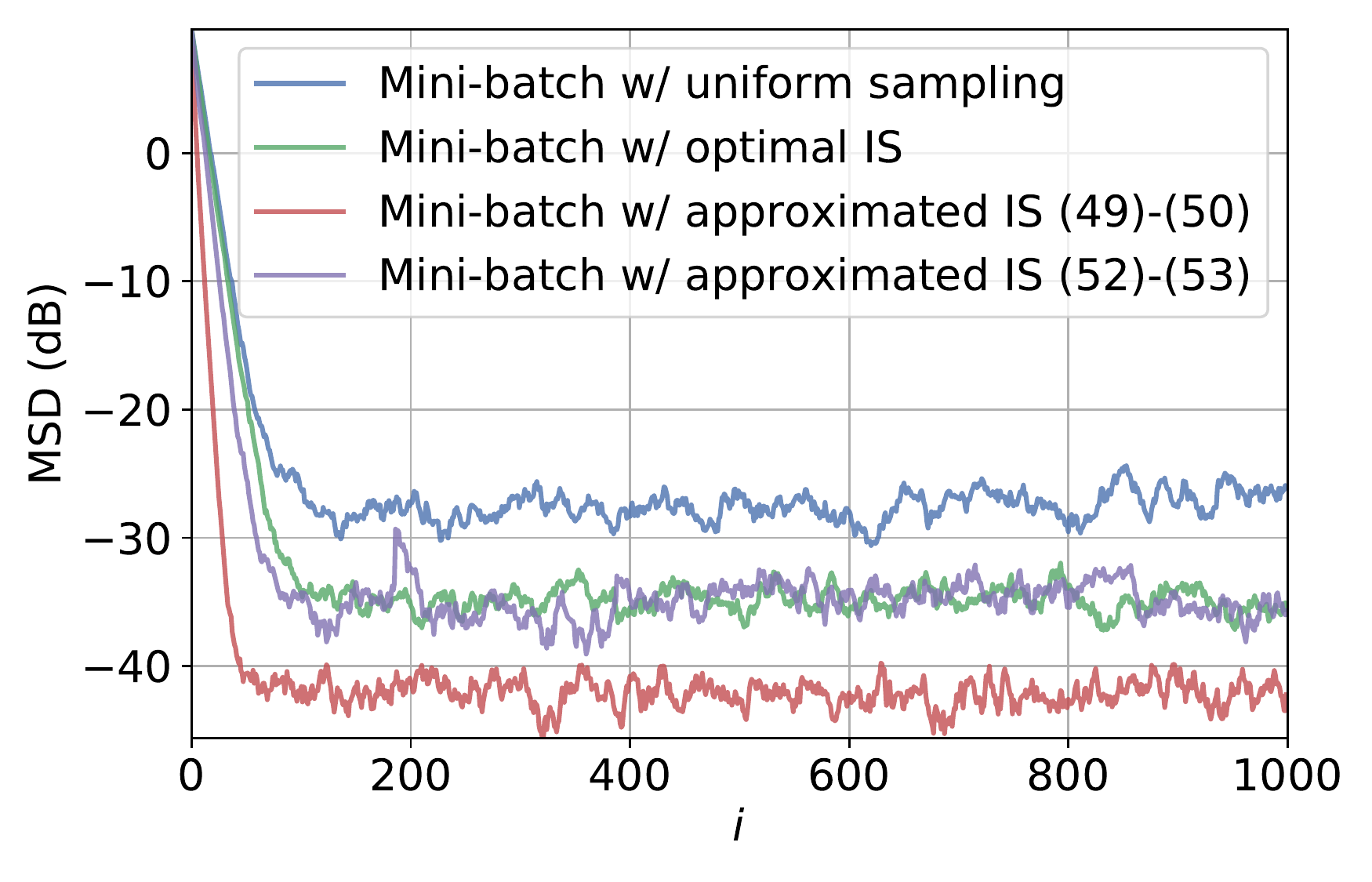}
	\caption{MSD plots of the regression problem: blue curve is the standard mini-batch implementation, green curve is the importance sampling implementation with the true probabilities, red curve is the importance sampling implementation with approximate probabilities \eqref{eq:approxPn1}--\eqref{eq:approxPk1}, purple curve is the importance sampling implementation with approximate probabilities \eqref{eq:approxPk}--\eqref{eq:approxPn}.}\label{fig:regMSD}
\end{center}
\end{figure}

\subsection{Classification}
We next study the theory in a classification context. We consider the ijcnn1 dataset \cite{ijcnn1}. The dataset consists of $35000$ training samples and $91701$ testing samples of $M = 22$ attributes. We distribute the data randomly in a non-IID fashion to $K = 100$ agents. Each agent receives a random number $N_k$ of data points, where $N_k$ ranges from 79 to 688. We run the two algorithms FedAvg and ISFedAvg. We set $\mu = 0.25$, $\rho = 0.0001$, $L = 10$, $B_k = 1$, and $E_k = 1$. We plot the testing error in Figure \ref{fig:classMSD}. We observe that importance sampling improves the testing error from $22.45\%$ to $18.46\%$. This is because importance sampling is more sample efficient.

\begin{figure}[h!]
\begin{center}
	\includegraphics[scale=0.5]{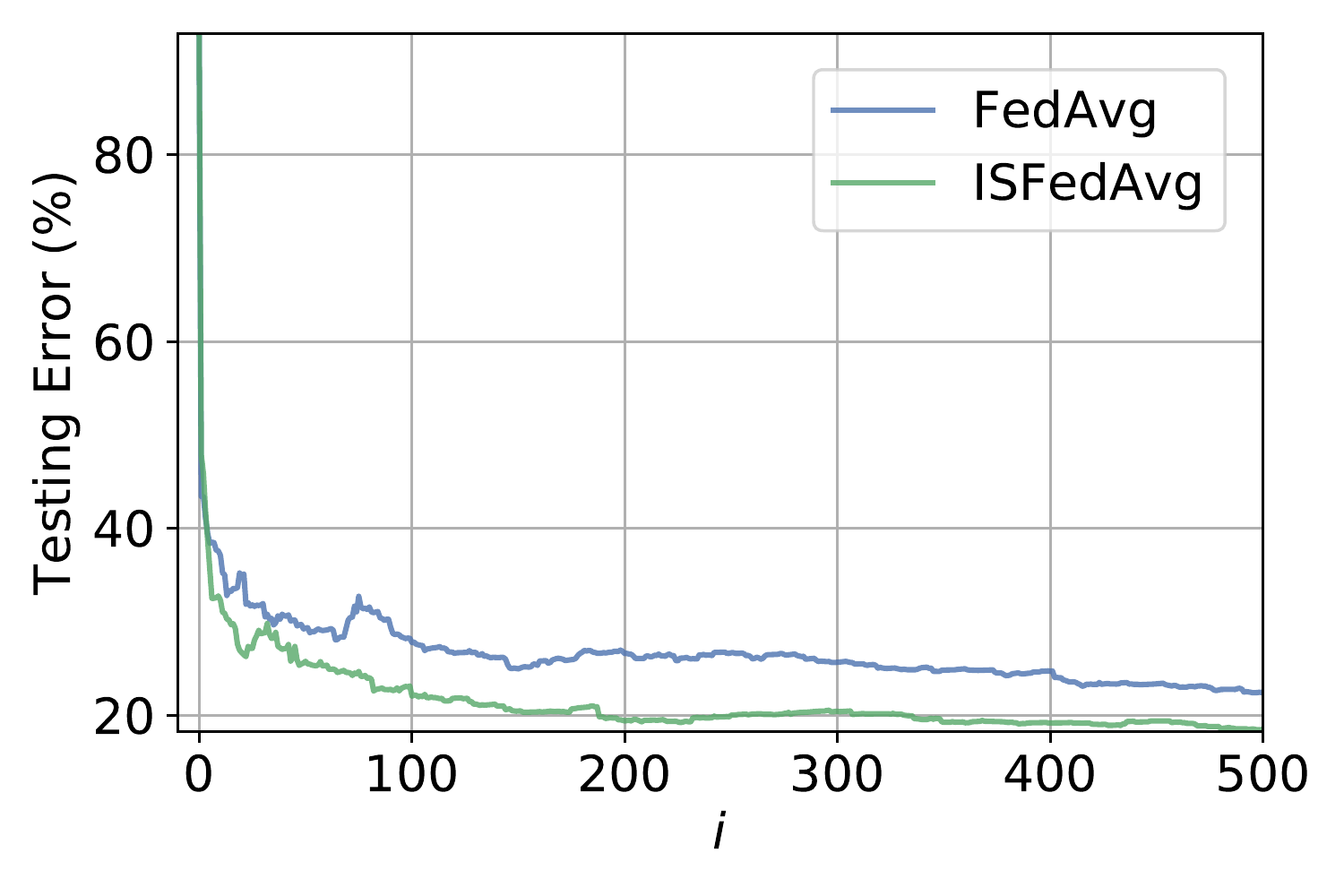}
	\caption{Testing error plots of the classification problem.
	}\label{fig:classMSD}
\end{center}
\end{figure}

\vspace{-1.2cm}
\section{Conclusion}
This work incorporates two levels of importance sampling into the operation of federated learning: one for selecting agents and another for selecting data batches at the agents. Optimal dynamical choices for the sampling probabilities are derived, and a detailed convergence analysis is performed. We also provided approximate expressions for the optimal sampling policies and illustrate the theoretical findings and the performance enhancement by means of simulations.   

\appendices
\section{Result on the Variance of the Mini-batch Estimate }\label{app:varMiniB}
We introduce the following auxiliary result that is a generalization of \cite{sampBook}.  Let \( \{ \mathcal{S} = \boldsymbol{x}_n \in \mathds{R}^M \}_{n=1}^N \) denote a set of \( N \) \emph{independent} random variables, each with mean \( \mathbb{E} \boldsymbol{x}_n = \overline{x}_n \) and variance \( \sigma_n^2 = \mathbb{E} {\| \boldsymbol{x}_n - \mathbb{E} \boldsymbol{x}_n \|}^2 \). We consider the problem of estimating the expected value of the sample mean:
\begin{equation}
	\overline{x} \triangleq \mathbb{E} \left( \frac{1}{N}\sum_{n=1}^N \boldsymbol{x}_n \right)
\end{equation}
We consider two estimators for \( \overline{x} \), both constructed by considering a mini-batch of samples, where \( \boldsymbol{x}_b^{\mathrm{r}} \) is constructed by sampling from \( \mathcal{S} \) \emph{with replacement}, and \( \boldsymbol{x}_b^{\mathrm{nr}} \) is sampled from \( \mathcal{S} \) \emph{without replacement}. Let $p_n$ be the normalized inclusion probability of $\bm{x}_n$. We then define the two estimators \eqref{eq:rEst} and \eqref{eq:nrEst}, and we would like to quantify the efficacy of these estimators in estimating \( \overline{x} \). 
\vspace{-0.1cm}
\begin{align}
	\widehat{\boldsymbol{x}}^{\mathrm{r}} &\triangleq \frac{1}{B} \sum_{b=1}^B \frac{1}{N p_b} \boldsymbol{x}_b^{\mathrm{r}}  \label{eq:rEst} \\
	\widehat{\boldsymbol{x}}^{\mathrm{nr}} &\triangleq \frac{1}{B} \sum_{b=1}^B \frac{1}{N p_b} \boldsymbol{x}_b^{\mathrm{nr}} \label{eq:nrEst}
\end{align}

\begin{lemma}[\textbf{Variance of the mini-batch mean with and without replacement}] \label{lem:varMiniB}
	Both estimators are unbiased and it holds that:
	\begin{align}
		\mathbb{E}{\widehat{\boldsymbol{x}}^{\mathrm{r}}} = \mathbb{E}{\widehat{\boldsymbol{x}}^{\mathrm{nr}}} &= \overline{x} ,\\
		\mathbb{E}{\| \widehat{\boldsymbol{x}}^{\mathrm{r}} - \overline{x} \|}^2  &= \frac{1}{B} \sum_{n=1}^N p_n \left(\frac{1}{N^2 p_n^2}\sigma_n^2 + {\left\| \frac{1}{Np_n}\overline{x}_n - \overline{x}\right\|}^2 \right), \\	
		\mathbb{E}{\left\| \widehat{\boldsymbol{x}}^{\mathrm{nr}} - \overline{x} \right\|}^2  &= \frac{1}{B} \sum_{n=1}^N  p_n\left( \frac{1}{N^2p_n^2}\sigma_n^2  +   \left\| \frac{1}{Np_n} \overline{x}_n - \overline{x}  \right\|^2 \right) \notag \\
		&+ \frac{1}{B^2}\sum_{n_1\neq n_2}\mathbb{P}( \mathbb{I}_{n_1}= 1, \mathbb{I}_{n_2}= 1) \notag \\
		&	\times \left(\frac{1}{Np_{n_1}}\overline{x}_{n_1} - \overline{x}\right)  \left(\frac{1}{Np_{n_2}}\overline{x}_{n_2} - \overline{x}\right),
	\end{align}
	where the notation $\mathbb{I}_n$ evaluates to one if $x_n\in{\cal B}^{nr}$ and is zero otherwise. 
\end{lemma}

\begin{proof}
	
	We begin with the \textit{with-replacement} setting. The randomness of the samples introduces some intricacies that need to be accounted for in the notation. For the mean, we have:
	\begin{align}
	\mathbb{E} \widehat{\boldsymbol{x}}^{\mathrm{r}} &= \frac{1}{B} \sum_{b=1}^B \mathbb{E} \left(\frac{1}{Np_b}\boldsymbol{x}_b^{\mathrm{r}}\right) = \frac{1}{B} \sum_{b=1}^B \mathbb{E}\left\{ \mathbb{E} \left\{  \frac{1}{Np_b} \boldsymbol{x}_b^{\mathrm{r}} \bigg| \mathcal{S} \right\} \right\} \notag \\
	&= \frac{1}{B} \sum_{b=1}^B \mathbb{E}\left\{ \sum_{n=1}^N p_n  \frac{1}{Np_n}\boldsymbol{x}_n \right\} = \frac{1}{B} \sum_{b=1}^B  \overline{x}  = \overline{x}.
	\end{align}
	For the variance we find:
	\begin{align}
	& \mathbb{E} {\left\| \widehat{\boldsymbol{x}}^{\mathrm{r}} - \overline{x} \right\|}^2 \notag \\
	&=\: \mathbb{E} {\left\| \frac{1}{B} \sum_{b=1}^B  \frac{1}{Np_b}\boldsymbol{x}_b^{\mathrm{r}} - \overline{x} \right\|}^2, \notag \\
	&=\: \mathbb{E} {\left\| \frac{1}{B} \sum_{b=1}^B \left( \frac{1}{Np_b} \boldsymbol{x}_b^{\mathrm{r}} - \overline{x} \right) \right\|}^2, \notag \\
	&=\:  \frac{1}{B^2} \sum_{b=1}^B \mathbb{E} {\left\| \frac{1}{Np_b} \boldsymbol{x}_b^{\mathrm{r}} - \overline{x} \right\|}^2 \notag \\
	&\quad+ \frac{1}{B^2} \sum_{b_1 \neq b_2} \mathbb{E} \left\{ {\left( \frac{1}{Np_{b_1}}\boldsymbol{x}_{b_1} - \overline{x} \right)} {\left( \frac{1}{Np_{b_2}}\boldsymbol{x}_{b_2} - \overline{x} \right)} \right\}, \notag \\
	&\stackrel{(a)}{=}\:  \frac{1}{B^2} \sum_{b=1}^B \mathbb{E} {\left\| \frac{1}{Np_b} \boldsymbol{x}_b^{\mathrm{r}} - \overline{x} \right\|}^2 \notag \\
	&\quad+ \frac{1}{B^2} \sum_{b_1 \neq b_2} \mathbb{E} {\left\{\frac{1}{Np_{b_1}} \boldsymbol{x}_{b_1} - \overline{x} \right\}} \mathbb{E}{\left\{\frac{1}{Np_{b_2}} \boldsymbol{x}_{b_2} - \overline{x} \right\}}, \notag \\
	&\stackrel{(b)}{=}\:  \frac{1}{B^2} \sum_{b=1}^B \mathbb{E} {\left\| \frac{1}{Np_b} \boldsymbol{x}_b^{\mathrm{r}} - \overline{x} \right\|}^2,
	\end{align}
	where \( (a) \) is a result of the fact that the elements of \( \mathcal{S} \) are independent and \( \boldsymbol{x}_b^{\mathrm{r}} \) is sampled from \( \mathcal{S} \) independently, and hence \( \boldsymbol{x}_{b_1} \) and \( \boldsymbol{x}_{b_2} \) are independent. Step \( (b) \) then follows from:
	\begin{equation}
	\mathbb{E} \left(\frac{1}{Np_b} \boldsymbol{x}_{b}\right) = \mathbb{E} \left(\frac{1}{N} \sum_{n=1}^N \boldsymbol{x}_n \right)= \overline{x}.
	\end{equation}
	Then, 
	\begin{align}
	& \mathbb{E} {\left\| \widehat{\boldsymbol{x}}^{\mathrm{r}} - \overline{x} \right\|}^2 \notag \\
	&{=}\:  \frac{1}{B^2} \sum_{b=1}^B \mathbb{E} {\left\| \frac{1}{Np_b} \boldsymbol{x}_b^{\mathrm{r}} - \overline{x} \right\|}^2, \notag \\
	&{=}\:  \frac{1}{B^2} \sum_{b=1}^B \mathbb{E} \left\{ \mathbb{E} {\left\| \frac{1}{Np_b} \boldsymbol{x}_b^{\mathrm{r}} - \overline{x} \right\|}^2 \bigg| \mathcal{S} \right\}, \notag \\
	&{=}\:  \frac{1}{B^2} \sum_{b=1}^B \mathbb{E} \left\{  \sum_{n=1}^N {p_n \left\| \frac{1}{Np_n} \boldsymbol{x}_n - \overline{x} \right\|}^2 \right\}, \notag \\
	&{=}\:  \frac{1}{ B^2} \sum_{b=1}^B \sum_{n=1}^N p_n\mathbb{E}{\left\| \frac{1}{Np_n}\boldsymbol{x}_n - \overline{x} \right\|}^2, \notag \\
	&{=}\:  \frac{1}{ B} \sum_{n=1}^N p_n \mathbb{E}{\left\| \frac{1}{Np_n} \boldsymbol{x}_n - \frac{1}{Np_n}\overline{x}_n + \frac{1}{Np_n}\overline{x}_n - \overline{x}\right\|}^2 ,\notag \\
	&{=}\:  \frac{1}{B} \sum_{n=1}^N p_n\left(\mathbb{E}{\left\| \frac{1}{Np_n} \boldsymbol{x}_n - \frac{1}{Np_n}\overline{x}_n \right\|}^2 + {\left\| \frac{1}{Np_n} \overline{x}_n - \overline{x}\right\|}^2 \right), \notag \\
	&{=}\:  \frac{1}{B} \sum_{n=1}^N p_n\left(\frac{1}{N^2p_n^2}\sigma_n^2 + {\left\| \frac{1}{Np_n}\overline{x}_n - \overline{x}\right\|}^2 \right) .
	\end{align}
	We now proceed to study the efficiency of the \textit{without replacement} mini-batch mean. The fact that the \( \boldsymbol{x}_b \) are sampled from \( \mathcal{S} \) without replacement causes pairs \( \boldsymbol{x}_{b_1}, \boldsymbol{x}_{b_2} \) to no longer be independent. We denote the set of points sampled from \( \mathcal{S} \) \textit{without replacement} by \( \mathcal{B}^{\mathrm{nr}} \) and introduce the activation function by:
	\begin{equation}
	\mathbb{I}_n \triangleq \begin{cases} 1, \ \mathrm{if}\ \boldsymbol{x}_n \in \mathcal{B}^{\mathrm{nr}}, \\ 0, \ \mathrm{if}\ \boldsymbol{x}_n \notin \mathcal{B}^{\mathrm{nr}}. \end{cases}
	\end{equation}
	Then, the estimator \( \widehat{\boldsymbol{x}}^{\mathrm{nr}}  \) can be written equivalently as:
	\begin{equation}
	\widehat{\boldsymbol{x}}^{\mathrm{nr}} = \frac{1}{B} \sum_{n=1}^N \mathbb{I}_n \frac{1}{Np_n} \boldsymbol{x}_n.
	\end{equation}
	For the mean, we have:
	\begin{align}
	\mathbb{E}\widehat{\boldsymbol{x}}^{\mathrm{nr}} &= \frac{1}{B} \sum_{n=1}^N \mathbb{E}\left\{ \mathbb{I}_n \frac{1}{Np_n} \boldsymbol{x}_n \right\} = \frac{1}{B} \sum_{n=1}^N \mathbb{E} \mathbb{I}_n \times \mathbb{E} \frac{1}{Np_n}\boldsymbol{x}_n \notag \\
	&= \frac{1}{B} \sum_{n=1}^N Bp_n \times \frac{1}{Np_n} \overline{x}_n = \frac{1}{N} \sum_{n=1}^N  \overline{x}_n = \overline{x}.
	\end{align}
	For the variance, we have:
	\begin{align}
	\mathbb{E} \left\| \widehat{\boldsymbol{x}}^{\mathrm{nr}} - \overline{x} \right\|^2 \notag 
	=&\: \mathbb{E} \left\| \frac{1}{B} \sum_{n=1}^N \mathbb{I}_n \left( \frac{1}{Np_n} \boldsymbol{x}_n - \overline{x} \right) \right\|^2, \notag \\
	=&\: \frac{1}{B^2} \sum_{n=1}^N \mathbb{E} \left\|\mathbb{I}_n \left( \frac{1}{Np_n} \boldsymbol{x}_n - \overline{x} \right)\right\|^2 \notag \\
	&+ \frac{1}{B^2} \sum_{n_1 \neq n_2} \mathbb{E} \Bigg\{  \mathbb{I}_{n_1} \left( \frac{1}{Np_{n_1}}\boldsymbol{x}_{n_1} - \overline{x} \right) \mathbb{I}_{n_2}   \notag \\ 
	&\qquad \qquad \qquad \qquad \times \left( \frac{1}{Np_{n_2}} \boldsymbol{x}_{n_2} - \overline{x} \right)  \Bigg\}.
	\end{align}
	We begin with:
	\begin{align}
	&\mathbb{E} \left\|\mathbb{I}_n \left( \frac{1}{Np_n}\boldsymbol{x}_n - \overline{x} \right) \right\|^2 \notag \\
	=&\: \mathbb{E} \left\{ \left\|\mathbb{I}_n \left( \frac{1}{Np_n} \boldsymbol{x}_n - \overline{x} \right) \right\|^2 \bigg| \mathbb{I}_n = 1 \right\} \times \mathds{P} \left( \mathbb{I}_n = 1 \right)\notag \\
	&+ \mathbb{E} \left\{ \left\|\mathbb{I}_n \left( \frac{1}{Np_n}\boldsymbol{x}_n - \overline{x} \right) \right\|^2 \bigg| \mathbb{I}_n = 0 \right\} \times \mathds{P} \left( \mathbb{I}_n = 0 \right) ,\notag \\
	=&\: Bp_n \left( \mathbb{E} \left\| \frac{1}{Np_n} \boldsymbol{x}_n - \frac{1}{Np_n}\overline{x}_n + \frac{1}{Np_n} \overline{x}_n - \overline{x} \right\|^2  \right) ,\notag \\
	=&\: Bp_n \left( \frac{1}{N^2p^2_n} \mathbb{E} \| \boldsymbol{x}_n - \overline{x}_n\|^2 + \left\|  \frac{1}{Np_n} \overline{x}_n - \overline{x} \right\|^2  \right), \notag \\
	=&\: Bp_n \left( \frac{1}{N^2p^2_n} \sigma_n^2 + \left\|  \frac{1}{Np_n}\overline{x}_n - \overline{x} \right\|^2  \right).
	\end{align}
	For the cross-term we have:
	\begin{align}
	&\: \mathbb{E} \left\{  \mathbb{I}_{n_1} \left( \frac{1}{Np_{n_1}} \boldsymbol{x}_{n_1} - \overline{x} \right)   \mathbb{I}_{n_2} \left( \frac{1}{Np_{n_2}} \boldsymbol{x}_{n_2} - \overline{x} \right)  \right\} \notag \\
	=&\: \mathbb{E} \bigg\{   \left( \frac{1}{Np_{n_1}} \boldsymbol{x}_{n_1} - \overline{x} \right)\left(  \frac{1}{Np_{n_2}}\boldsymbol{x}_{n_2} - \overline{x} \right) \bigg| \mathbb{I}_{n_1} =1, \mathbb{I}_{n_2}=1 \bigg\} \notag \\
	& \times\mathds{P}\left( \mathbb{I}_{n_1} = 1, \mathbb{I}_{n_2} = 1 \right)
, \notag \\
	=&\: \mathbb{P} \left( \mathbb{I}_{n_2} =1, \mathbb{I}_{n_1} = 1 \right)\left( \frac{1}{Np_{n_1}}\mathbb{E} \boldsymbol{x}_{n_1} - \overline{x} \right) \left( \frac{1}{Np_{n_2}}\mathbb{E} \boldsymbol{x}_{n_2} - \overline{x} \right),  \notag \\
	=&\: \mathbb{P} \left(  \mathbb{I}_{n_2} =1, \mathbb{I}_{n_1} = 1 \right)  \left( \frac{1}{Np_{n_1}}\overline{x}_{n_1} - \overline{x} \right) \left( \frac{1}{Np_{n_2}}\overline{x}_{n_2} - \overline{x} \right) .
	\end{align}
	We then get the desired result.
\end{proof}

We note the following bound on the variance of sampling without replacement estimator. Using Jensen's inequality, we can get rid of the cross-term, and deduce from (26) the following inequality:
	\begin{equation}\label{eq:bdVarWOR}
		\mathbb{E} \left\| \widehat{\boldsymbol{x}}^{\mathrm{nr}} - \overline{x} \right\|^2 \leq \sum_{n=1}^N p_n \left ( \frac{1}{N^2p_n^2}\sigma_n^2 + \left\Vert \frac{1}{Np_n}\overline{x}_n - \overline{x}\right\Vert^2 \right).
	\end{equation}

\section{Proof of Lemma \ref{lemm:gradNoise}}\label{app:gradNoise}
\begin{proof}
		We start with the sampling \textit{with-replacement} construction. We have $K$ agents from which we sample $L$. Thus, $N$ and $B$ in Lemma \ref{lem:varMiniB} are $K$ and $L$, respectively. Also:
	\begin{align}
		\bm{x}_{k} &=  \widehat{\grad{w}P_k}(\w_{i-1}) ,	\\
		\overline{x}_k &= \grad{w}P_k(\w_{i-1}), 	\\
		\overline{x} &= \frac{1}{K} \sum_{k = 1}^K \grad{w}P_k(\w_{i-1}).
	\end{align}
	Then $\sigma_k^2$, which quantifies the second order moment of the local gradient noise,  becomes:
	\begin{align}
		\sigma_k^2 &= \mathbb{E} \left\{	\left \Vert \widehat{\grad{w}P_k}(\w_{i-1}) - \grad{w}P_k(\w_{i-1}) \right\Vert^2  \bigg| \w_{i-1} \right\},	\notag\\
		&=  \frac{1}{E_k^2 B_k^2} \sum_{e=1}^{E_k}\sum_{b \in \B{k}{e}} \mathbb{E} \Bigg\{ \bigg\Vert \frac{1}{N_kp_b^{(k)}} \grad{w}Q_k(\w_{i-1};\bm{x}_{k,b}) 	\notag \\
		& \qquad \qquad \qquad \qquad \qquad \quad- \grad{w}P_k(\w_{i-1})	\bigg\Vert^2 \bigg| \w_{i-1} \Bigg\}, \notag \\
		&= \frac{1}{E_kB_k^2} \sum_{b \in \B{k}{e}} \mathbb{E} \Bigg\{ \bigg\Vert \frac{1}{N_kp_b^{(k)}} \grad{w}Q_k(\w_{i-1};\bm{x}_{k,b}) \notag \\
		& - \frac{1}{N_kp_b^{(k)}}\grad{w}Q_k(w^o;\bm{x}_{k,b}) + \frac{1}{N_kp_b^{(k)}}\grad{w}Q_k(w^o;\bm{x}_{k,b}) \notag \\
		& - \grad{w}P_k(w^o)  + \grad{w}P_k(w^o)- \grad{w}P_k(\w_{i-1})	\bigg\Vert^2 \bigg | \w_{i-1} \Bigg\}, 	\notag  \\
		&\stackrel{(a)}{\leq} \frac{3}{E_kB_k^2} \sum_{b \in \B{k}{e}} \Bigg\{\mathbb{E} \Bigg \{ \bigg \Vert \frac{1}{N_kp_b^{(k)}} \grad{w}Q_k(\w_{i-1};\bm{x}_{k,b}) \notag \\
		&\qquad \qquad \qquad \qquad- \frac{1}{N_kp_b^{(k)}}\grad{w}Q_k(w^o;\bm{x}_{k,b})  \bigg\Vert^2 \bigg| \w_{i-1} \Bigg\} 	\notag \\
		&+ \mathbb{E} \left \{ \left \Vert \frac{1}{N_kp_b^{(k)}} \grad{w}Q_k(w^o;\bm{x}_{k,b}) - \grad{w}P_k(w^o) \right\Vert^2 \bigg| \w_{i-1} \right\}  \notag \\
		&+ \mathbb{E} \left\{ \left \Vert \grad{w}P_k(w^o)-\grad{w}P_k(\w_{i-1}) \right\Vert^2 \bigg| \w_{i-1} \right\} \Bigg\}, \notag \\
		&= \frac{3}{E_kB_k^2}\sum_{b \in \B{k}{e}} \Bigg\{ \sum_{n=1}^{N_k} p_n^{(k)} \bigg\Vert \frac{1}{N_kp_n^{(k)}} \grad{w}Q_k(\w_{i-1};x_{k,n}) 	\notag \\
		& \qquad \qquad \qquad \qquad \qquad- \frac{1}{N_kp_n^{(k)}} \grad{w}Q_k(w^o;x_{k,n})	\bigg\Vert^2 \notag \\
		& + \sum_{n=1}^{N_k}p_n^{(k)} \left \Vert \frac{1}{N_kp_n^{(k)}} \grad{w}Q_k(w^o;x_{k,n}) - \grad{w}P_k(w^o) \right\Vert^2 \notag \\
		& + 	\left \Vert \grad{w}P_k(w^o)-\grad{w}P_k(\w_{i-1}) \right\Vert^2 \Bigg\},	\notag \\
		&\stackrel{(b)}{\leq} \frac{3}{E_kB_k^2}\sum_{b \in \B{k}{e}} \Bigg\{ \left(1+\sum_{n=1}^{N_k} \frac{1}{N_k^2p_n^{(k)}}\right)\delta^2 \Vert \widetilde{\w}_{i-1} \Vert^2  \notag 
		\end{align}
	\begin{align}
		&+ \sum_{n=1}^{N_k} p_n^{(k)}\left \Vert \frac{1}{N_kp_n^{(k)}} \grad{w}Q_k(w^o;x_{k,n}) - \grad{w}P_k(w^o) \right\Vert^2   \Bigg\},\notag \\
		&\stackrel{(c)}{\leq} \frac{3\delta^2}{E_kB_k} \left( 1 + \frac{1}{N_k^2 } \sum_{n=1}^{N_k}\frac{1}{p_n^{(k)}} \right)\Vert \widetilde{\w}_{i-1}\Vert^2 \notag \\
		&+ \frac{6}{E_kB_kN_k^2} \sum_{n=1}^{N_k} \frac{1}{ p_n^{(k)}} \left\Vert  \grad{w}Q_k(w^o;x_{k,n})\right\Vert^2 \notag  \\
		&+ \frac{6}{E_kB_k}\left\Vert \grad{w}P_k(w^o)\right\Vert^2 \notag \\
		&= \beta_{s,k}^2 \Vert \widetilde{\w}_{i-1}\Vert^2 + \sigma_{s,k}^2 + \frac{6}{E_kB_k}\left\Vert \grad{w}P_k(w^o)\right\Vert^2 ,
	\end{align}		 
	where $(a)$ and $(c)$ follow from using Jensen's inequality, and $(b)$ follows from using the $\delta-$Lipschitz property of the gradients.  
	Thus, using Lemma \ref{lem:varMiniB}, we bound the stochastic noise variance as follows:
	\begin{align}
		&\mathbb{E}\left \{ \Vert \bm{s}_i \Vert^2 | \w_{i-1}\right\} = \frac{1}{L} \sum_{k=1}^K p_k \Bigg\{ \frac{1}{K^2 p_k^2}\sigma_k^2   \notag \\
		& + \bigg \Vert \frac{1}{Kp_k} \grad{w}P_k(\w_{i-1})- \frac{1}{K} \sum_{\ell=1}^K \grad{w}P_{\ell}(\w_{i-1}) \bigg\Vert^2 \Bigg\}.
	\end{align}
	We focus on the second term since the first term has already been bounded. Using Jensen's inequality in $(a)$ and Lipschitz condition of the gradients in $(b)$, we get:
	\begin{align}
		&\left \Vert \frac{1}{Kp_k} \grad{w}P_k(\w_{i-1}) - \frac{1}{K} \sum_{\ell=1}^K \grad{w}P_{\ell}(\w_{i-1}) \right\Vert^2  \notag \\
		&= \frac{1}{K^2}\bigg\Vert \frac{1}{p_k} \grad{w}P_k(\w_{i-1}) -\frac{1}{p_k} \grad{w}P_k(w^o) + \frac{1}{p_k} \grad{w}P_k(w^o) \notag \\
		& +  \sum_{\ell=1}^K \grad{w}P_{\ell}(w^o)- \sum_{\ell=1}^K \grad{w}P_{\ell}(\w_{i-1}) \bigg\Vert^2, \notag \\
		&\stackrel{(a)}{\leq} \frac{3}{K^2p^2_k}  \left\Vert \grad{w}P_k(\w_{i-1}) - \grad{w}P_k(w^o) \right\Vert^2   \notag \\
		&+ \frac{3}{K^2p^2_k} \left\Vert \grad{w}P_k(w^o)\right \Vert^2 \notag \\
		&+  \frac{3}{K}\sum_{\ell=1}^K\left \Vert  \grad{w}P_{\ell}(w^o)- \grad{w}P_{\ell}(\w_{i-1}) \right\Vert^2,  \notag \\
		&\stackrel{(b)}{\leq}  3\delta^2 \left(1+\frac{1}{K^2 p_k^2}\right)\Vert \widetilde{\w}_{i-1}\Vert^2 + \frac{3}{K^2p^2_k} \left\Vert \grad{w}P_k(w^o)\right \Vert^2.
	\end{align}
	Then, putting things together, we get:
	
	\begin{align}
		&\mathbb{E}\left \{ \Vert \bm{s}_i \Vert^2 | \w_{i-1}\right\} \notag \\
		&\leq \frac{1}{L}\sum_{k=1}^K p_k \Bigg\{\frac{\beta_{s,k}^2}{K^2p_k^2}   \Vert \widetilde{\w}_{i-1}\Vert^2 +  \frac{1}{K^2p_k^2}  \sigma_{s,k}^2  \notag \\ 
		&\qquad  \frac{6}{K^2p_k^2 E_kB_k} \Vert \grad{w}P_k(w^o)\Vert^2+ \frac{3}{K^2p^2_k} \left\Vert \grad{w}P_k(w^o)\right \Vert^2 \notag \\
		&\qquad  + 3\delta^2 \left(1+\frac{1}{K^2 p_k^2}\right)\Vert \widetilde{\w}_{i-1}\Vert^2  \Bigg\}, \notag \\
		&= \frac{1}{L}\sum_{k=1}^K \left(  \frac{\beta_{s,k}^2}{K^2p_k} + 3\delta^2 p_k + \frac{3\delta^2}{K^2 p_k} \right) \Vert \widetilde{\w}_{i-1}\Vert^2 \notag \\
		&+ \frac{1}{LK^2} \sum_{k=1}^K \frac{1}{p_k} \left\{ \sigma_{s,k}^2 + \left(3+\frac{6}{E_kB_k}\right)\Vert \grad{w}P_k(w^o)\Vert^2 \right\}, \notag \\
		&= \left( \frac{3\delta^2}{L}+\frac{1}{LK^2}\sum_{k=1}^K \frac{1}{p_k}\left(\beta_{s,k}^2 + 3\delta^2 \right)\right) \Vert \widetilde{\w}_{i-1}\Vert^2\notag \\
		&+ \frac{1}{LK^2} \sum_{k=1}^K \frac{1}{p_k} \left\{ \sigma_{s,k}^2 +\left(3+\frac{6}{E_kB_k}\right)\Vert \grad{w}P_k(w^o)\Vert^2 \right\}, \notag \\ 
		&= \beta_s^2 \Vert \widetilde{\w}_{i-1}\Vert^2 + \sigma_s^2 .
	\end{align}
	
	Next, we move to the sampling \textit{without replacement} construction. The variance $\sigma_k^2$ becomes:
	\begin{align}
		\sigma_k^2 &= \mathbb{E} \left\{	\left \Vert \widehat{\grad{w}P_k}(\w_{i-1}) - \grad{w}P_k(\w_{i-1}) \right\Vert^2  \bigg| \w_{i-1} \right\},	\notag\\
		&=  \mathbb{E} \Bigg\{	\bigg\Vert \frac{1}{E_kB_k}\sum_{e=1}^{E_k}\sum_{n=1}^{N_k}\mathbb{I}_n\frac{1}{N_kp_n^{(k)}} \grad{w}Q_k(\w_{i-1}; \bm{x}_n)  \notag \\
		&\qquad \quad- \grad{w}P_k(\w_{i-1}) \bigg\Vert^2  \bigg| \w_{i-1} \Bigg\},	\notag\\
		&=  \frac{1}{E_k^2 }\sum_{e=1}^{E_k} \mathbb{E} \Bigg\{ \bigg\Vert\frac{1}{B_k}\sum_{n=1}^{N_k}\mathbb{I}_n \frac{1}{N_kp_n^{(k)}} \grad{w}Q_k(\w_{i-1};\bm{x}_n) 	\notag \\
		&\qquad \qquad\qquad\quad- \grad{w}P_k(\w_{i-1}) \bigg\Vert^2  \bigg| \w_{i-1} \Bigg\},	\notag\\
		&= \frac{1}{E_kB_k^2} \sum_{n=1}^{N_k} \mathbb{E}\Bigg \{ \bigg \Vert \mathbb{I}_n\frac{1}{N_kp_n^{(k)}} \grad{w}Q_k(\w_{i-1};\bm{x}_n) \notag \\
		&\qquad\qquad \qquad \quad- \grad{w}P_k(\w_{i-1})  \bigg\Vert^2 \bigg| \w_{i-1} \Bigg\}, \notag \\
		&+ \frac{1}{E_kB_k^2}\sum_{n_1\neq n_2} \mathbb{E}\Bigg\{ \mathbb{I}_{n_1}\bigg( \frac{1}{N_kp_{n_1}^{(k)}} \grad{w}Q_k(\bm{w}_{i-1};\bm{x}_{n_1}) \notag \\
		&- \grad{w}P_k(\w_{i-1}) \bigg) \mathbb{I}_{n_2} \bigg( \frac{1}{N_kp_{n_2}^{(k)}} \grad{w}Q_k(\bm{w}_{i-1};\bm{x}_{n_2}) \notag \\
		&- \grad{w}P_k(\w_{i-1}) \bigg) \bigg| \w_{i-1}\Bigg\}.
	\end{align}
	Starting with the first term, we use Jensen's inequality in $(a)$ and $(c)$ and the Lipschitz condition in $(b)$ to get:
	
	\begin{align}
		&\frac{1}{E_kB_k^2}\sum_{n=1}^{N_k} \mathbb{P}(\mathbb{I}_n = 1)\mathbb{E}\Bigg \{ \bigg \Vert\frac{1}{N_kp_n^{(k)}} \grad{w}Q_k(\w_{i-1};x_{k,n}) \notag \\
		&\qquad\qquad- \grad{w}P_k(\w_{i-1})  \bigg\Vert^2 \bigg| \w_{i-1},  \mathbb{I}_n=1 \Bigg\}, \notag \\
		&= \frac{1}{E_kB_k}\sum_{n=1}^{N_k}p_n^{(k)} \bigg \Vert\frac{1}{N_kp_n^{(k)}} \grad{w}Q_k(\w_{i-1};x_{k,n}) \notag \\
		&\qquad\qquad\qquad \quad- \grad{w}P_k(\w_{i-1})  \bigg\Vert^2 , \notag \\
		&= \frac{1}{E_kB_k}\sum_{n=1}^{N_k}p_n^{(k)} \bigg\Vert \frac{1}{N_kp_n^{(k)}} \grad{w}Q_k(\w_{i-1};x_{k,n}) \notag \\
		& - \frac{1}{N_kp_n^{(k)}} \grad{w}Q_k(w^o;x_{k,n}) + \frac{1}{N_kp_n^{(k)}} \grad{w}Q_k(w^o;x_{k,n})  \notag \\
		& -  \grad{w}P_k(w^o) +  \grad{w}P_k(w^o) -  \grad{w}P_k(\w_{i-1}) \bigg\Vert^2 , \notag \\
		&\stackrel{(a)}{\leq}  \frac{3}{E_kB_k}\sum_{n=1}^{N_k} \frac{1}{N_k^2p_n^{(k)}} \Vert \grad{w}Q_k(\w_{i-1};x_{k,n}) \notag \\
		&- \grad{w}Q_k(w^o;x_{k,n})\Vert^2 \notag \\
		&+ p_n^{(k)}	 \left\Vert \frac{1}{N_kp_n^{(k)}} \grad{w}Q_k(w^o;x_{k,n}) -  \grad{w}P_k(w^o)  \right\Vert^2  \notag \\
		&+ p_n^{(k)} \left \Vert \grad{w}P_k(w^o) -  \grad{w}P_k(\w_{i-1}) \right\Vert^2, \notag \\
		&=\frac{3}{E_kB_k}\sum_{n=1}^{N_k}\Bigg\{ \frac{1}{N_k^2p_n^{(k)}} \Vert \grad{w}Q_k(\w_{i-1};x_{k,n}) \notag \\
		&- \grad{w}Q_k(w^o;x_{k,n})\Vert^2 \notag \\
		&+ p_n^{(k)}	 \left\Vert \frac{1}{N_kp_n^{(k)}} \grad{w}Q_k(w^o;x_{k,n}) -  \grad{w}P_k(w^o)  \right\Vert^2  \Bigg\} \notag \\
		&+ \frac{3}{E_kB_k}\left \Vert \grad{w}P_k(w^o) -  \grad{w}P_k(\w_{i-1}) \right\Vert^2, \notag \\
		&\stackrel{(b)}{\leq} \frac{3\delta^2}{E_kB_k} \left( 1+\frac{1}{N_k^2}\sum_{n=1}^{N_k}\frac{1}{p_n^{(k)}} \right) \Vert \widetilde{\w}_{i-1}\Vert^2  \notag \\
		&+ \frac{3}{E_kB_k}\sum_{n=1}^{N_k}p_n^{(k)} \left\Vert \frac{1}{N_kp_n^{(k)}} \grad{w}Q_k(w^o;x_{k,n}) -  \grad{w}P_k(w^o)  \right\Vert^2 ,\notag \\ 
		&\stackrel{(c)}{\leq} \frac{3\delta^2}{E_kB_k} \left( 1+\frac{1}{N_k^2}\sum_{n=1}^{N_k}\frac{1}{p_n^{(k)}} \right) \Vert \widetilde{\w}_{i-1}\Vert^2  \notag \\ 
		&+ \frac{6}{E_kB_k} \sum_{n=1}^{N_k}\frac{1}{N_k^2p_n^{(k)}}\Vert \grad{w}Q_k(w^o;x_{k,n})\Vert^2 \notag \\
		&+ \frac{6}{E_kB_k}\Vert \grad{w}P_k(w^o)\Vert^2, \notag \\
		&= \beta_{s,k}^2 \Vert \widetilde{\w}_{i-1}\Vert^2 + \sigma_{s,k}^2 + \frac{6}{E_kB_k} \Vert \grad{w}P_k(w^o)\Vert^2,
	\end{align}
	{\color{blue}The cross-term reduces to 0 by first conditionig over $\mathbb{I}_{n_1} = 1, \mathbb{I}_{n_2} = 1$ and then splittinng the expectation. Each of the two terms are zero.}
	Thus, putting everything together, we get:
	\begin{align}
		\sigma_k^2 \leq \beta_{s,k}^2 \Vert \widetilde{\w}_{i-1}\Vert^2 + \sigma_{s,k}^2 + \frac{6}{E_kB_k} \Vert \grad{w}P_k(w^o)\Vert^2.
	\end{align}
	Next, to bound the second order moment of the gradient noise, we use \eqref{eq:bdVarWOR}:
	\begin{align}
		&\mathbb{E}\left\{ \Vert \bm{s}_i \Vert^2 | \w_{i-1} \right\}  \leq \sum_{k=1}^K p_k \Bigg( \frac{1}{K^2 p_k^2} \sigma_k^2  \notag \\
		&+ \bigg\Vert \frac{1}{Kp_k} \grad{w}P_k(\w_{i-1}) - \frac{1}{K} \sum_{\ell=1}^K \grad{w}P_{\ell}(\w_{i-1}) \bigg\Vert^2 \Bigg). 
	\end{align}
	The first term is of the same form as for sampling with replacement, and thus can be bounded similarly: 
	\begin{align}
		&\mathbb{E}\{\Vert \bm{s}_i \Vert^2 | \w_{i-1}\}  \leq  \sum_{k=1}^K p_k\Bigg\{ \frac{\beta_{s,k}^2}{K^2p_k^2}\Vert \widetilde{\w}_{i-1}\Vert^2 + \frac{1}{K^2p_k^2}\sigma_{s,k}^2 \notag \\
		&+ \frac{6}{K^2p_k^2E_kB_k}\Vert \grad{w}P_k(w^o)\Vert^2  + \frac{3}{K^2p_k^2}\Vert \grad{w}P_k(w^o)\Vert^2  \notag \\
		&+   3\delta^2 \left(1+\frac{1}{K^2p_k^2}\right) \Vert \widetilde{\w}_{i-1} \Vert^2 \Bigg\}  \notag \\
		&= \beta_s^2\Vert \widetilde{\w}_{i-1}\Vert ^2 +  \sigma_s^2.
	\end{align}
\end{proof}

\section{Proof of Lemma \ref{lemm:convCentSol}}\label{app:convCentSol}

\begin{proof}
	We first note the following result by using $\frac{1}{K}\sum_{k=1}^K \grad{w}P_k(w^o) = 0$:
	\begin{align}\label{eq:prf-step}
		&\left\Vert \widetilde{\w}_{i-1} + \mu \frac{1}{K}\sum_{k=1}^K \grad{w}P_k(\w_{i-1}) \right\Vert^2 \notag \\
		&= \Vert \widetilde{\w}_{i-1}\Vert^2 + \mu^2 \left\Vert \frac{1}{K}\sum_{k=1}^K \grad{w}P_k(w^o)-\grad{w}P_k(\w_{i-1}) \right\Vert^2 \notag \\
		& \quad+ 2\mu \widetilde{\w}_{i-1}^\tran \frac{1}{K}\sum_{k=1}^K \grad{w}P_k(\w_{i-1}), \notag \\ 
		& \stackrel{(a)}{\leq}  \Vert \widetilde{\w}_{i-1}\Vert^2 + \mu^2 \frac{1}{K}\sum_{k=1}^K \left\Vert  \grad{w}P_k(w^o)-\grad{w}P_k(\w_{i-1}) \right\Vert^2 \notag \\
		& \quad+ 2\mu \widetilde{\w}_{i-1}^\tran \frac{1}{K}\sum_{k=1}^K \grad{w}P_k(\w_{i-1}), \notag \\ 
		&\stackrel{(b)}{\leq} (1+\mu^2 \delta^2) \Vert \widetilde{\w}_{i-1}\Vert^2 + 2\mu  \widetilde{\w}_{i-1}^\tran \frac{1}{K}\sum_{k=1}^K \grad{w}P_k(\w_{i-1}), \notag \\ 
		&\stackrel{(c)}{\leq} (1+\mu^2 \delta^2) \Vert \widetilde{\w}_{i-1}\Vert^2 \notag \\
		&\quad + 2\mu \frac{1}{K}\sum_{k=1}^K \left( P_k(w^o)- P_k(\w_{i-1}) - \frac{\nu}{2}\Vert \widetilde{\w}_{i-1}\Vert^2 \right), \notag \\ 
		&\stackrel{(d)}{\leq} (1+\mu^2 \delta^2) \Vert \widetilde{\w}_{i-1}\Vert^2 -2\mu \frac{1}{K}\sum_{k=1}^K \nu \Vert \widetilde{\w}_{i-1}\Vert^2 , \notag \\
		&= (1-2\mu\nu + \mu^2\delta^2)\Vert  \widetilde{\w}_{i-1}\Vert^2,
	\end{align}
	where $(a)$ follows from Jensen's inequality, $(b)$ from the Lipschitz condition, and $(c)$ and $(d)$ from strong convexity condition.
	
	Returning to the main expression:
	\begin{equation}
		\widetilde{\w}_{i-1} + \mu \frac{1}{K}\sum_{k=1}^K \grad{w}P_k(\w_{i-1})+ \mu \bm{s}_i,
	\end{equation} 
 and taking  conditional expectations, we obtain:
	\begin{align}
		& \mathbb{E} \left\{ \left\Vert \widetilde{\w}_{i-1} + \mu \frac{1}{K}\sum_{k=1}^K \grad{w}P_k(\w_{i-1})+ \mu \bm{s}_i \right\Vert^2 \Bigg | \w_{i-1} \right\} , \notag \\
		&\stackrel{(a)}{=} \mathbb{E} \left\{  \left\Vert \widetilde{\w}_{i-1} + \mu \frac{1}{K}\sum_{k=1}^K \grad{w}P_k(\w_{i-1}) \right\Vert^2 \big | \w_{i-1}\right\} \notag \\
		& \quad+ \mu^2 \mathbb{E} \left\{ \Vert \bm{s}_i\Vert^2 \big | \w_{i-1} \right\} , \notag \\
		&\stackrel{(b)}{\leq} (1-2\mu\nu + \mu^2\delta^2)\Vert  \widetilde{\w}_{i-1}\Vert^2 \notag \\
		& \quad + \mu^2 \left( \beta_s^2 \Vert \widetilde{\w}_{i-1} \Vert^2 + \eta_s \Vert \widetilde{\w}_{i-1} \Vert + \sigma_s^2 \right) , 
	\end{align}
	where the cross-term in $(a)$ is zero because of the zero mean property of the gradient noise, and $(b)$ follows from \eqref{eq:prf-step} and using the bound on the second order moment of the gradient noise.
	
	Next, taking expectation again to remove the conditioning we get:
	\begin{align}
		&\mathbb{E}\left\Vert \widetilde{\w}_{i-1}+ \mu \frac{1}{K}\sum_{k=1}^K \grad{w}P_k(\w_{i-1})+ \mu \bm{s}_i\right\Vert ^2 \notag \\
		&\leq \left( 1-2\mu\nu + \mu^2(\delta^2+\beta_s^2)\right) \mathbb{E}\Vert  \widetilde{\w}_{i-1}\Vert^2 + \mu^2 \eta_s \mathbb{E}\Vert \widetilde{\w}_{i-1} \Vert   \notag \\
		&\quad + \mu^2 \sigma_s^2 .
	\end{align}
\end{proof}

\section{Proof of Lemma \ref{lem:locGradNoise}}\label{app:locGradNoise}
\begin{proof}
	To show the mean is zero, it is enough to calculate the mean of the approximate gradient. We start with the sampling with replacement scheme where the samples are chosen independently from each other:
	\begin{align}
		&\mathbb{E}\left \{ \frac{1}{B_k}\sum_{b\in \B{k}{e}} \frac{1}{N_kp_b^{(k)}} \grad{w}Q_k(\w_{k,e-1};\bm{x}_{k,b}) \big| \mathcal{F}_{e-1}, \Li \right\} \notag\\
		&= \frac{1}{B_k} \sum_{b \in \B{k}{e}} \mathbb{E}\left \{\frac{1}{N_k p_b^{(k)}}  \grad{w}Q_k(\w_{k,e-1};\bm{x}_{k,b}) \big| \mathcal{F}_{e-1}, \Li \right\} , \notag\\
		&= \frac{1}{B_k} \sum_{b \in \B{k}{e}} \sum_{n=1}^{N_k} \frac{1}{N_k p_n^{(k)}}  \grad{w}Q_k(\w_{k,e-1};\bm{x}_{k,n}) , \notag \\
		&= \frac{1}{N_k}\sum_{n=1}^{N_k} \grad{w}Q_k(\w_{k,e-1};\bm{x}_{k,n}).
	\end{align}
	As for the sampling without replacement scheme, since the samples are now dependent, we introduce the indicator function $\mathbb{I}_n$ and the derivation goes as follows:
	\begin{align}
		&\mathbb{E}\left \{ \frac{1}{B_k}\sum_{b\in \B{k}{e}} \frac{1}{N_kp_b^{(k)}} \grad{w}Q_k(\w_{k,e-1};\bm{x}_{k,b}) \big| \mathcal{F}_{e-1}, \Li \right\} \notag\\
		&= \mathbb{E}\left \{ \frac{1}{B_k} \sum_{n=1}^{N_k}  \frac{\mathbb{I}_n}{N_k p_n^{(k)}}  \grad{w}Q_k(\w_{k,e-1};\bm{x}_{k,n}) \big| \mathcal{F}_{e-1}, \Li \right\} , \notag\\
		&= \frac{1}{B_k}\sum_{n=1}^{N_k} \frac{\mathbb{P}(\mathbb{I}_n = 1)}{N_k p_n^{(k)}}  \grad{w}Q_k(\w_{k,e-1};\bm{x}_{k,n}) , \notag \\
		&= \frac{1}{N_k}\sum_{n=1}^{N_k} \grad{w}Q_k(\w_{k,e-1};\bm{x}_{k,n}).
	\end{align}
	Next, to bound the second order moment, we start with an intermediate step and bound the second order moment of the individual gradient noise of one sample. The derivation below holds regardless of the sampling scheme. By adding and subtracting $\frac{1}{N_kp_n^{(k)}}\grad{w}Q_k(w_k^o;\bm{x}_{k,n})$, adding $\grad{w}P_k(w_k^o) = 0$, and then using Jensen's inequality and Lipschitz condition, we get:
	\begin{align}
		&\left\Vert \frac{1}{N_kp_n^{(k)}} \grad{w}Q_k(\w_{k,e-1};\bm{x}_{k,n}) - \grad{w}P_k(\w_{k,e-1}) \right\Vert^2 \notag \\
		&\leq 3 \Bigg\Vert \frac{1}{N_kp_n^{(k)}} \grad{w}Q_k(\w_{k,e-1};\bm{x}_{k,n}) \notag \\
		&\quad - \frac{1}{N_kp_n^{(k)}}\grad{w}Q_k(w_k^o;\bm{x}_{k,n})  \Bigg\Vert^2\notag \\
		&\quad +3\left\Vert \frac{1}{N_kp_n^{(k)}}\grad{w}Q_k(w_k^o;\bm{x}_{k,n}) \right\Vert^2  + 3\delta^2 \Vert \widetilde{\w}_{k,e-1}\Vert^2.
	\end{align}
	Then, taking the conditional expectation and using the Lipschitz property, we get:
	\begin{align}
		&\mathbb{E}\Bigg\{  \Bigg\Vert \frac{1}{N_kp_n^{(k)}} \grad{w}Q_k(\w_{k,e-1};\bm{x}_{k,n}) \notag \\
		&\quad - \grad{w}P_k(\w_{k,e-1}) \Bigg\Vert^2 \Bigg| \mathcal{F}_{e-1}, \Li \Bigg\} \notag \\
		&\leq \sum_{n=1}^{N_k} \frac{3p_n^{(k)}}{N_k^2\left(p_n^{(k)}\right)^2} \big(\Vert \grad{w}Q_k(\w_{k,e-1};x_{k,n}) \notag \\
		&\quad -\grad{w}Q_k(w_{k}^o;x_{k,n})\Vert^2+ \left\Vert\grad{w}Q_k(w_k^o;x_{k,n}) \right\Vert^2 \big) \notag \\
		&\quad+ 3\delta^2 \Vert \widetilde{\w}_{k,e-1}\Vert^2, \notag \\
		&\leq \sum_{n=1}^{N_k} \frac{3}{N_k^2p_n^{(k)}} \left(\delta^2\Vert \widetilde{\w}_{k,e-1}\Vert^2 +\left\Vert\grad{w}Q_k(w_k^o;x_{k,n}) \right\Vert^2 \right) \notag \\
		&\quad +  3\delta^2 \Vert \widetilde{\w}_{k,e-1}\Vert^2.
	\end{align}
	Now going back to calculating the second order moment of the local incremental gradient noise, we first start with the sampling with replacement. Using the fact that the samples are independent we get:
	\begin{align}
		&\mathbb{E}\left\{ \Vert \bm{q}_{k,i,e} \Vert^2 \big| \mathcal{F}_{e-1}, \Li\right\}\notag \\
		&= \frac{1}{K^2p_k^2B_k^2}\sum_{b \in \B{k}{e}} \mathbb{E}\Bigg\{  \Bigg\Vert \frac{1}{N_kp_b^{(k)}} \grad{w}Q_k(\w_{k,e-1};\bm{x}_{k,b}) \notag \\
		&\quad - \grad{w}P_k(\w_{k,e-1}) \Bigg\Vert^2 \Bigg| \mathcal{F}_{e-1}, \Li \Bigg\} \notag \\
	&\leq \frac{3\delta^2}{K^2p_k^2B_k} \left( 1  + \frac{1}{N_k^2}\sum_{n=1}^{N_k} \frac{1}{p_n^{(k)}} \right) \Vert \widetilde{ \w}_{k,e-1}\Vert^2 \notag \\
		&+ \frac{3}{K^2p_k^2 B_kN_k^2}\sum_{n=1}^{N_k} \frac{1}{p_n^{(k)}}\Vert \grad{w}Q_k(w_k^o;x_{k,n})\Vert^2.
	\end{align}
	As for the sampling without replacement, we also introduce the indicator function and write out the square of sums. The cross-terms disappear since each term has zero mean. The derivation then follows similarly to that of the sampling with replacement. More formally:
	\begin{align}
	&\mathbb{E}\left\{ \Vert \bm{q}_{k,i,e} \Vert^2 \big| \mathcal{F}_{e-1}, \Li\right\}\notag \\
		&= \frac{1}{K^2p_k^2B_k^2}\sum_{n=1}^{N_k}\mathbb{P}(\mathbb{I}_n = 1)\mathbb{E}\Bigg\{  \Bigg\Vert \frac{1}{N_kp_n^{(k)}} \grad{w}Q_k(\w_{k,e-1};\bm{x}_{k,n}) \notag \\
		&\quad - \grad{w}P_k(\w_{k,e-1}) \Bigg\Vert^2 \Bigg| \mathbb{I}_n = 1,\mathcal{F}_{e-1}, \Li \Bigg\} \notag \\ 
		&= \frac{1}{K^2p_k^2B_k}\sum_{n=1}^{N_k}p_n^{(k)}  \Bigg\Vert \frac{1}{N_kp_n^{(k)}} \grad{w}Q_k(\w_{k,e-1};x_{k,n}) \notag \\
		&\quad - \grad{w}P_k(\w_{k,e-1}) \Bigg\Vert^2 \notag \\ 
		&\leq \frac{3\delta^2}{K^2p_k^2B_k} \left( 1  + \frac{1}{N_k^2}\sum_{n=1}^{N_k} \frac{1}{p_n^{(k)}} \right) \Vert \widetilde{ \w}_{k,e-1}\Vert^2 \notag \\
		&+ \frac{3}{K^2p_k^2 B_kN_k^2}\sum_{n=1}^{N_k} \frac{1}{p_n^{(k)}}\Vert \grad{w}Q_k(w_k^o;x_{k,n})\Vert^2.
	\end{align}
	
\end{proof}

\section{Proof of Lemma \ref{lem:von-loc-inc-step}}\label{app:von-loc-inc-step}

\begin{proof}
	We subtract $w_k^o$ from both sides of \eqref{eq:localUp} and use  \eqref{eq:locGradNoise} to get:
	\begin{align}\label{eq:lem-prf-locErrRec}
		\widetilde{\w}_{k,e} &= \widetilde{\w}_{k,e-1} + \mu \grad{w}P_k(\w_{k,e-1}) + \mu \bm{q}_{k,i,e}.
	\end{align}
	We bound the first two terms and use the fact that $\grad{w}P_k(w_k^o) = 0$, Lipschitz condition, and the convexity of the cost function:
	\begin{align}
		&\left\Vert \widetilde{\w}_{k,e-1} + \mu \grad{w}P_k(\w_{k,e-1}) \right\Vert^2 \notag \\
		&\quad+ \mu^2 \Vert \grad{w}P_k(\w_{k,e-1})\Vert^2 , \notag\\
		&= \Vert \widetilde{\w}_{k,e-1}\Vert^2 + 2\mu\widetilde{\w}_{k,e-1}^\tran \grad{w}P_k(\w_{k,e-1}) \notag \\
		&\quad+ \mu^2 \Vert \grad{w}P_k(w_{k}^o)- \grad{w}P_k(\w_{k,e-1})\Vert^2 , \notag\\
		&\leq (1+\mu^2 \delta^2)\Vert \widetilde{\w}_{k,e-1}\Vert^2 + 2\mu\widetilde{\w}_{k,e-1}^\tran \grad{w}P_k(\w_{k,e-1}), \notag \\
		&\leq (1-2\nu\mu + \mu^2 \delta^2) \Vert \widetilde{\w}_{k,e-1}\Vert^2.
	\end{align}
	Returning to \eqref{eq:lem-prf-locErrRec}, squaring both sides, conditioning on the filtration $\mathcal{F}_{e-1}$, and taking expectations we obtain:
	\begin{align}
		&\mathbb{E} \left\{ \Vert \widetilde{\w}_{k,e} \Vert^2 \big| \mathcal{F}_{e-1} \right\} \notag \\
		&\stackrel{(a)}{=}   \mathbb{E} \left\{ \Vert \widetilde{\w}_{k,e-1} + \mu \grad{w}P_k(\w_{k,e-1}) \Vert^2 \big| \mathcal{F}_{e-1} \right\} \notag \\
		&\quad + \mu^2 \mathbb{E} \left\{ \Vert \bm{q}_{k,i,e}^2 \Vert^2 \big| \mathcal{F}_{e-1} \right\} , \notag \\
		&\leq \left (1-2\nu\mu + \mu^2 \left(\delta^2 + \frac{E_k}{K^2p_k^2}\beta^2_{s,k}\right) \right) \Vert \widetilde{\w}_{k,e-1}\Vert^2  \notag \\
		&\quad + \mu^2 \frac{1}{K^2p_k^2}\sigma_{q,k}^2,
	\end{align}
	where the cross term in $(a)$ is zero because of the zero mean property of the local incremental gradient noise. Taking expectations on both sides again removes the condition on the filtration and leads to the desired result. By further iterating recursion \eqref{eq:thrm-loc-MSD-rec} we obtain:
	\begin{equation}
		\mathbb{E}\Vert \widetilde{\w}_{k,e}\Vert^2 \leq \lambda_{k}^{e} \mathbb{E}\Vert \widetilde{\w}_{k,0}\Vert^2 + \frac{1-\lambda_k^e}{1-\lambda_k}\mu^2\sigma_{q,k}^2.
	\end{equation}	 
\end{proof}  

\section{Proof of Lemma \ref{lemm:incNoise}} \label{app:incNoise}
\begin{proof}
	
	First, using Jensen's inequality $(a)$ and Lipschitz continuity $(b)$, we obtain:
	\begin{align}\label{eq:bd-qi}
		\Vert \bm{q}_i \Vert^2 &\stackrel{(a)}{\leq} \frac{1}{L}\sum_{\ell \in \Li} \frac{1}{K^2p_{\ell}^2E_{\ell}B_{\ell}} \sum_{e=1}^{E_{\ell}} \sum_{b\in \B{\ell}{e}} \frac{1}{N^2_{\ell}\left (p_{b}^{(\ell)} \right)^2}  \notag \\
		&\quad \times\left\Vert\grad{w}Q_{\ell}(\w_{\ell,e-1};\bm{x}_{\ell,b}) - \grad{w}Q_{\ell}(\w_{i-1};\bm{x}_{\ell,b}) \right\Vert^2, \notag \\
		&\stackrel{(b)}{\leq} \frac{\delta^2}{L}\sum_{\ell \in \Li} \frac{1}{K^2p_{\ell}^2E_{\ell}B_{\ell}} \sum_{e=1}^{E_{\ell}} \sum_{b\in \B{\ell}{e}} \frac{\Vert \w_{i-1} - \w_{\ell,e-1}\Vert^2}{N_{\ell}^2 \left(p_{b}^{(\ell)} \right)^2}.
	\end{align}
	Next, we focus on $\Vert \w_{i-1}-\w_{\ell,e-1}\Vert^2$, and by applying Jensen's inequality in $(a)$ and $(b)$ and Lipschitz condition in $(c)$ we obtain:
	
	\begin{align}
		&\Vert \w_{i-1}-\w_{\ell,e-1}\Vert^2 \notag \\
		&= \mu^2 \left\Vert \frac{1}{E_{\ell}B_{\ell}}\sum_{f=0}^{e-2} \sum_{b\in \B{\ell}{f}} \frac{1}{N_{\ell}p_b^{(\ell)}} \grad{w}Q_{\ell}(\w_{\ell,f};\bm{x}_{\ell,b}) \right\Vert^2, \notag \\
		&\stackrel{(a)}{\leq} \frac{\mu^2}{E_{\ell}B_{\ell}}\sum_{f=0}^{e-2} \sum_{b\in \B{\ell}{f}} \frac{1}{N_{\ell}^2 \left(p_b^{(\ell)}\right)^2} \Vert  \grad{w}Q_{\ell}(\w_{\ell,f};\bm{x}_{\ell,b}) \notag \\
		& \quad - \grad{w}Q_{\ell}(w_{\ell}^o;\bm{x}_{\ell,b}) + \grad{w}Q_{\ell}(w_{\ell}^o;\bm{x}_{\ell,b}) \Vert^2, \notag \\
		&\stackrel{(b)}{\leq} \frac{2\mu^2}{E_{\ell}B_{\ell}}\sum_{f=0}^{e-2} \sum_{b\in \B{\ell}{f}} \frac{1}{N_{\ell}^2 \left(p_b^{(\ell)}\right)^2} \big( \Vert  \grad{w}Q_{\ell}(\w_{\ell,f};\bm{x}_{\ell,b})\notag \\
		& \quad - \grad{w}Q_{\ell}(w_{\ell}^o;\bm{x}_{\ell,b})\Vert^2 + \Vert\grad{w}Q_{\ell}(w_{\ell}^o;\bm{x}_{\ell,b}) \Vert^2 \big), \notag \\
		&\stackrel{(c)}{\leq} \frac{2\mu^2}{E_{\ell}B_{\ell}}\sum_{f=0}^{e-2} \sum_{b\in \B{\ell}{f}} \frac{1}{N_{\ell}^2 \left(p_b^{(\ell)}\right)^2} \Big( \delta^2\Vert  \widetilde{\w}_{\ell,f}\Vert^2\notag \\
		& \quad + \Vert\grad{w}Q_{\ell}(w_{\ell}^o;\bm{x}_{\ell,b}) \Vert^2 \Big). \notag  
	\end{align}
	Then, taking the expectation given the previous filtration $\mathcal{F}_{e-2}$ and the participating agents $\Li$, we see that:
	\begin{align}
		&\mathbb{E}\left\{ \Vert \w_{i-1} - \w_{\ell,e-1}\Vert^2 \big| \mathcal{F}_{e-2}, \Li \right\}  \notag \\
		&\leq \frac{2\mu^2\delta^2}{E_{\ell}B_{\ell}} \sum_{f=0}^{e-2} \Vert \widetilde{\w}_{\ell,f}\Vert^2 \mathbb{E}\left\{ \sum_{b \in \B{\ell}{f}} \frac{1}{N_{\ell}^2 \left( p_b^{(\ell)}\right)^2} \Bigg| \mathcal{F}_{e-2}, \Li \right\} \notag \\
		&\quad +  \frac{2\mu^2}{E_{\ell}B_{\ell}} \sum_{f=0}^{e-2} \mathbb{E}\left\{ \sum_{b\in \B{\ell}{f}} \frac{\Vert \grad{w}Q_{\ell}(w_{\ell}^o;\bm{x}_{\ell,b})\Vert^2}{N_{\ell}^2 \left(p_b^{(\ell)}\right)^2}  \Bigg| \mathcal{F}_{e-2}, \Li \right\} \\
		&= \frac{2\mu^2\delta^2}{E_{\ell}B_{\ell}} \sum_{f=0}^{e-2} \Vert \widetilde{\w}_{\ell,f}\Vert^2 \sum_{n=1}^{N_{\ell}} \frac{\mathbb{P}(\mathbb{I}_n=1)}{N_{\ell}^2 \left( p_n^{(\ell)}\right)^2} \notag \\
		&\quad + \frac{2\mu^2}{E_{\ell}B_{\ell}} \sum_{f=0}^{e-2}\sum_{n=1}^{N_{\ell}} \frac{\mathbb{P}(\mathbb{I}_n = 1)}{N_{\ell}^2 \left( p_n^{(\ell)}\right)^2} \Vert \grad{w}Q_{\ell}(w_{\ell}^o;x_{\ell,n})\Vert^2,  \notag \\
		&=\frac{2\mu^2 \delta^2}{E_{\ell}} \sum_{f=0}^{e-2} \Vert \widetilde{\w}_{\ell,f}  \Vert^2 \sum_{n=1}^{N_{\ell}}\frac{1}{N_{\ell}^2  p_n^{(\ell)}}+  \frac{2\mu^2  (e-1)}{3E_{\ell}} \sigma_{q,\ell}^2.
	\end{align}
	Then, taking expectation again over the filtration, we obtain:
	\begin{align}
		&\mathbb{E} \left\{ \Vert \w_{i-1} - \w_{\ell,e-1}\Vert^2 \big| \Li \right\} \notag \\
		&\leq \frac{2\mu^2\delta^2}{E_{\ell}} \sum_{f=0}^{e-1} \mathbb{E} \left\{ \Vert \widetilde{\w}_{\ell,f}\Vert ^2 \big| \Li \right\} \sum_{n=1}^{N_{\ell}} \frac{1}{N_{\ell}^2 p_n^{(\ell)}} + \frac{2\mu^2  (e-1)}{3E_{\ell}}\sigma_{q,\ell}^2,  \notag \\
		&\stackrel{(a)}{\leq} 	\frac{2\mu^2\delta^2}{E_{\ell}} \sum_{n=1}^{N_{\ell}} \frac{1}{N_{\ell}^2 p_n^{(\ell)}} \sum_{f=0}^{e-1} \Bigg( \lambda_{\ell}^f \mathbb{E} \left\{ \Vert \widetilde{\w}_{\ell,0}\Vert ^2 \big| \Li \right\}  \notag \\
		&\quad+ \frac{\mu^2}{K^2p_{\ell}^2}\frac{1-\lambda_{\ell}^f}{1-\lambda_{\ell}}\sigma_{q,\ell}^2 \Bigg) + \frac{2\mu^2(e-1)}{3E_{\ell}}\sigma_{q,\ell}^2, \notag \\
		&= \frac{2\mu^2\delta^2}{E_{\ell}} \sum_{n=1}^{N_{\ell}} \frac{1}{N_{\ell}^2 p_n^{(\ell)}} \Bigg( \frac{1-\lambda_{\ell}^e}{1-\lambda_{\ell}}\mathbb{E} \left\{ \Vert \widetilde{\w}_{\ell,0}\Vert ^2 \big| \Li \right\} \notag \\
		&\quad +\frac{\mu^2}{K^2p_k^2}\frac{e(1-\lambda_{\ell}) - 1 + \lambda_{\ell}^e}{(1-\lambda_{\ell})^2}\sigma_{q,\ell}^2 \Bigg)+\frac{2\mu^2(e-1)}{3E_{\ell}}\sigma_{q,\ell}^2, \notag \\
		&\stackrel{(b)}{\leq} \frac{2\mu^2\delta^2}{E_{\ell}} \sum_{n=1}^{N_{\ell}} \frac{1}{N_{\ell}^2 p_n^{(\ell)}} \Bigg( 2\frac{1-\lambda_{\ell}^e}{1-\lambda_{\ell}}\mathbb{E} \left\{ \Vert \widetilde{\w}_{i-1}\Vert ^2 \big| \Li \right\}  \notag \\
		&\quad + 2\frac{1-\lambda_{\ell}^e}{1-\lambda_{\ell}}\mathbb{E} \left\{ \Vert w^o - w_{\ell}^o \Vert ^2 \big| \Li \right\} \notag \\
		&\quad +\frac{\mu^2}{K^2p_k^2}\frac{e(1-\lambda_{\ell}) - 1 + \lambda_{\ell}^e}{(1-\lambda_{\ell})^2}\sigma_{q,\ell}^2 \Bigg)+\frac{2\mu^2(e-1)}{3E_{\ell}}\sigma_{q,\ell}^2,
	\end{align}
	where we used Lemma \ref{lem:von-loc-inc-step} in $(a)$, and in $(b)$ we added and subtracted $w^o$ and used Jensen's inequality. Then, summing over $e$ results in:
	\begin{align}
		&\frac{1}{E_{\ell}}\sum_{e=1}^{E_{\ell}} \mathbb{E}\left \{ \Vert \w_{i-1}-\w_{\ell,e-1}\Vert^2 \big| \Li \right\} \notag \\
		&\leq \frac{2\mu^2\delta^2}{E_{\ell}^2}\sum_{n=1}^{N_{\ell}} \frac{1}{N_{\ell}^2 p_n^{(\ell)}} \Bigg(2\frac{(E_{\ell}+1)(1-\lambda_{\ell})-1+\lambda_{\ell}^{E_{\ell}+1}}{(1-\lambda_{\ell})^2}   \notag \\
		&\quad \times \left( \mathbb{E} \left\{ \Vert \widetilde{\w}_{i-1}\Vert ^2 \big| \Li \right\} +\mathbb{E} \left\{ \Vert w^o - w_{\ell}^o \Vert ^2 \big| \Li \right\} \right) \notag \\
		&\quad + \frac{E_{\ell}(E_{\ell}+1)(1-\lambda_{\ell})^2 - 2E_{\ell}(1-\lambda_{\ell}) + 2\lambda_{\ell} - 2\lambda_{\ell}^{E_{\ell}+1}}{(1-\lambda_{\ell})^3} \notag \\
		&\quad \times \frac{\mu^2}{K^2p_k^2} \sigma_{q,\ell}^2\Bigg)+ \frac{E_{\ell}(E_{\ell}-1)\mu^2}{3E_{\ell}^2}\sigma_{q,\ell}^2.
	\end{align}
	Taking the expectation of \eqref{eq:bd-qi} given the choice of the agents and plugging the above expression, we get:
	\begin{align}
		&\mathbb{E}\left\{ \Vert \bm{q}_i \Vert^2 \big| \Li \right\} \notag \\
		&\leq \frac{\delta^2}{L}\sum_{\ell \in \Li} \frac{1}{K^2p_{\ell}^2}  \Big(a\mu^2 \mathbb{E} \left\{ \Vert \widetilde{\w}_{i-1}\Vert ^2 \big| \Li \right\} \notag \\
		&\quad  + a\mu^2 \mathbb{E} \left\{ \Vert w^o - w_{\ell}^o \Vert ^2 \big| \Li \right\}\notag \\
		&\quad + (b \mu^4  + c \mu^2 )\sigma_{q,\ell}^2\Big) \frac{1}{B_{\ell}N_{\ell}^2}\mathbb{E}\left \{ \sum_{b\in \B{\ell}{e}} \frac{1}{\left(p_{b}^{(\ell)} \right)^2} \Bigg| \Li \right\}, \notag \\
		&= \frac{\delta^2}{L}\sum_{\ell \in \Li} \frac{1}{K^2p_{\ell}^2}  \Big(a\mu^2 \mathbb{E} \left\{ \Vert \widetilde{\w}_{i-1}\Vert ^2 \big| \Li \right\} \notag \\
		&\quad + a\mu^2 \mathbb{E} \left\{ \Vert w^o - w_{\ell}^o \Vert ^2 \big| \Li \right\}+ (b \mu^4  + c \mu^2 )\sigma_{q,\ell}^2\Big)\sum_{n=1}^{N_{\ell}}\frac{1}{N_{\ell}^2p_n^{(\ell)}},
	\end{align}
	where we introduced constants $a,b,c$ to make the notation simpler.
	Then, taking again the expectation to remove the conditioning and using Assumption \ref{assum:bdLocalMin}:
	\begin{align}
		\mathbb{E}\Vert \bm{q}_i \Vert^2 &\leq \delta^2\sum_{k=1}^K \frac{1}{K^2p_{k}} \sum_{n=1}^{N_k}\frac{1}{N_{k}^2p_n^{(k)}} \Big(a\mu^2 \mathbb{E}  \Vert \widetilde{\w}_{i-1}\Vert ^2  \notag \\
		&\quad + a\mu^2  \Vert w^o - w_{k}^o \Vert ^2 + (b \mu^4  + c \mu^2 )\sigma_{q,k}^2\Big) \notag \\
		& \leq  \delta^2\sum_{k=1}^K \frac{1}{K^2p_{k}} \sum_{n=1}^{N_k}\frac{1}{N_{k}^2p_n^{(k)}} \Big(a\mu^2 \mathbb{E}  \Vert \widetilde{\w}_{i-1}\Vert ^2 + a\mu^2  \xi^2 \notag \\
		&\quad + (b \mu^4  + c \mu^2 )\sigma_{q,k}^2\Big) .
	\end{align}	
	Further simplifying the notation gives us the desired result. Thus, since $a = O(\mu^{-1})$, $b = O(\mu^{-2})$ and $c = O(1)$, we get $\mathbb{E}\Vert \bm{q}_i\Vert^2 = O(\mu)$.
\end{proof}

\bibliographystyle{IEEEtran}
{\balance{\bibliography{SPAWC-refs}}}

\end{document}